\documentclass[journal,twoside]{IEEEtran}

\usepackage{amsmath,amssymb,amsfonts,amsthm} 
\usepackage{algorithmic}
\usepackage{algorithm}
\usepackage{array}
\usepackage[caption=false,font=normalsize,labelfont=sf,textfont=sf]{subfig}
\usepackage{textcomp}
\usepackage{stfloats}
\usepackage{url}
\usepackage{verbatim}
\usepackage{graphicx}
\usepackage{cite}
\usepackage{textcase}       
\usepackage{booktabs}       
\usepackage{color,colortbl} 
\usepackage{multirow}       
\usepackage{cuted}          

\newtheorem{theorem}{Theorem}

\theoremstyle{definition}
\newtheorem{remark}{Remark}
\newtheorem{definition}{Definition}

\begin{document}

\title{Self-Supervised Topologically Invariant Manifold Learning for Railway Image Quality Assessment}

\author{Tingqiong~Cui,
        Yibu~Yang,
        Yang~Li,
        Jiahao~Fu,
        Xiaoliu~Luo,
        Xu~Wang,
        Mengzhu~Wang,
        Siyuan~Liu,
        Guanghui~Huang%
\thanks{This work was supported by CRRC Chongqing Co., Ltd. under Project ``Development of an Artificial Intelligence Large Model for Urban Transportation Vehicle Exterior Styling and Industrial Design'' (Grant No. 2024CCA572). \emph{(Corresponding author: Guanghui Huang.)}}%
\thanks{Tingqiong Cui, Yibu Yang, Yang Li, and Jiahao Fu are with CRRC Chongqing Co., Ltd., Chongqing 400000, China (e-mail: cuitingqiong.cqgs@crrcgc.cc}%
\thanks{Xiaoliu Luo is with the School of Mathematical Sciences,Chongqing University of Technology, Chongqing 400054, China (e-mail: luoxiaoliu@cqut.edu.cn).}%
\thanks{Xu Wang is with the College of Computer Science, Sichuan University, Chengdu 610065, China (e-mail: wangxu.scu@gmail.com).}%
\thanks{Mengzhu Wang is with the School of Artificial Intelligence, Hebei University of Technology, China (e-mail: dreamkily@gmail.com).}%
\thanks{Siyuan Liu is with the State Key Laboratory of Mechanical Transmission for Advanced Equipment, Chongqing University, Chongqing 400044, China (e-mail: syliu@cqu.edu.cn).}%
\thanks{Guanghui Huang is with the College of Mathematics and Statistics, Chongqing University, Chongqing 401331, China (e-mail: hgh@cqu.edu.cn).}%
}

\markboth{SUBMITTED FOR REVIEW}%
{Huang \textit{et al.}: Self-Supervised Topologically Invariant Manifold Learning}

\maketitle

\begin{abstract}
Existing blind image quality assessment (BIQA) methods typically rely on synthetic distortions and subjective annotations, limiting generalization in real-world domains. To address this, we propose a fully self-supervised BIQA framework based on topologically invariant manifold learning under boundary constraints, which constructs a stable quality reference without manual labels. The framework generates progressive background dilution scales via repeated random cropping around each target; exploiting the monotonic degradation of target information density across these scales, it establishes a self-constrained quality manifold. A linearized spatial moment projection eliminates geometric distortions from random cropping; then a monotonicity divergence filter prunes background-sensitive evaluators, isolating an elite pool \(\mathcal{M}_{\text{elite}}\). A robust M-estimator with a principal component stabilizer fuses the metrics into an asymptotically efficient pseudo-ground truth \(q_{\text{PGT}}\), contracting variance toward the Cram\'{e}r-Rao lower bound. Extensive evaluations demonstrate that the elite evaluator pool, distilled from 11 baseline metrics, secures superior zero-shot transferability across standard synthetic and wild benchmarks (CSIQ, LIVEC, LIVE-2). Concurrently, deployments on the CQU Railway Rolling Stock Surveillance Dataset (2,797 images) yield a manifold cosine similarity \(>0.999\) and a 100.0\% survival rate under industrial extreme stresses, robustly validating its cross-paradigm decoupling and topological resilience.
\end{abstract}


\section{Introduction}
\label{sec:introduction}

Existing blind image quality assessment (BIQA) frameworks rely on large-scale synthetic distortions and intensive human subjective visual annotations to construct perceptual evaluation mappings~\cite{zhai2020perceptual}. However, in real-world deployment domains such as industrial railway rolling stock surveillance, human annotation suffers from prohibitive labor overhead and unreproducible cognitive variance.

To bypass these limitations, recent works explore three complementary directions: feature decoupling under localized patch processing~\cite{xiang2026learning}, generative noise separation via state-space models~\cite{lan2025no}, and Gaussian mixture distribution modeling~\cite{gao2025blind}.

Nevertheless, these next-generation methods still introduce systematic artifacts and bottlenecks in target-oriented vision scenarios. Specifically, capturing fixed-position industrial objects frequently introduces highly varying framing scales, aspect-ratio fluctuations, and frame-centroid shifts. Heterogeneous quality evaluators possess highly asymmetric and unsynchronized sensitivities to these background composition styles and spatial edge-clipping perturbations. Without structural regularization, these spatial confounders inject mixed-sign cross-covariances into the covariance matrix. When the covariance matrix fails to maintain non-negativity and irreducibility, as required by the Perron-Frobenius conditions~\cite{meyer2000matrix}, the loading descriptors undergo destructive spectral loading cancellation within the dominant eigen-subspace. This cancellation causes the explained consensus variance to decay and fractures the predictive authority of the quality scale anchor.

To resolve these limitations, this paper proposes a closed-loop self-supervised BIQA framework based on topologically invariant manifold learning under boundary constraints~\cite{isomap2000}. By exploiting the monotonic degradation of target information density across progressive background dilution scales, the framework establishes a self-constrained quality manifold. From this manifold, we extract an unbiased pseudo-ground truth \(q_{\text{PGT}}\) without requiring references or pseudo-reference generation~\cite{min2018blind,lin2018hallucinated}. Structurally, we combine high-order spatial moment orthogonal projection layers~\cite{flusser2016moments} with a dynamic performance-driven monotonicity divergence filter to isolate the invariant target quality core from chaotic contextual clutters. This purified covariance matrix guarantees a unique, non-degenerate dominant axis that absorbs the collective consensus truth without information cross-cancellation. Furthermore, a robust Welsch/Leclerc-type M-estimator~\cite{leclerc1989constructing} combined with a non-parametric Median Absolute Deviation (MAD) scale matrix~\cite{rousseeuw1993alternatives} is formulated to guarantee numerical continuity toward the optimal Cram\'{e}r-Rao lower bound~\cite{kay1993fundamentals}, ensuring high fault tolerance and system survival under long-tail industrial edge failure scenarios.

To guarantee full community-wide verification, the entire operational code stack and anonymized score tensors have been made publicly accessible at \url{https://github.com/cqugege/Self-Supervised-BIQA-Manifold}. The core technical contributions of this work are summarized as follows:

\begin{enumerate}
    \item \textbf{Invariant Manifold Learning}: A self-supervised framework that constructs a quality manifold from background dilution scales without manual labels, providing a transferable anchor for out-of-distribution deployments.
    \item \textbf{Spectral Purification}: A linearized spatial moment projection combined with a monotonicity divergence filter that isolates an elite evaluator pool, ensuring Perron-Frobenius spectral alignment and preventing spectral loading cancellation.
    \item \textbf{Robust M-Estimation}: A Leclerc-type M-estimator with a principal component stabilizer that guarantees asymptotically efficient pseudo-ground truth generation and topological invariance, validated by the Davis-Kahan theorem~\cite{davis1970rotation}.
\end{enumerate}

\section{Related Works}
\label{sec:related_work}

The evolution of blind image quality assessment (BIQA) algorithms fundamentally tracks the shifting paradigms of perceptual feature representation in computer vision, migrating from low-level statistical priors to high-level semantic embeddings~\cite{zhai2020perceptual}.  We review these paradigms chronologically—from hand-crafted NSS to deep supervised models and vision-language methods—and identify their shared limitations in target-oriented scenarios, motivating the proposed framework.

\subsection{Hand-Crafted Natural Scene Statistics (NSS) and Early BIQA Paradigms}
\label{subsec:nss_review}

The classical foundation of non-reference image quality criteria rests upon the hypothesis that pristine natural images exhibit strict statistical regularities, which are highly vulnerable to distortions. Pioneering frameworks, such as BRISQUE~\cite{brisque}, utilize locally normalized luminance coefficients via generalized Gaussian distributions (GGD). To bypass the requirement of synthetic distortion labels during training, unsupervised architectures like NIQE~\cite{niqe} fit a multivariate Gaussian (MVG) model over premium corpus repositories, quantifying quality degradation as the Mahalanobis distance between the target and reference distributions. However, these hand-crafted models operate on pixel-level features and lack high-level understanding. In real-world surveillance scenes, their statistical regularities fracture due to geometric co-confounders, sacrificing reliability.

\subsection{Data-Driven Supervised Deep Perception Models}
\label{subsec:deep_review}

Data-driven CNNs have largely superseded hand-crafted features by learning hierarchical representations from massive databases. Frameworks such as NIMA~\cite{nima} leverage deep features to output a continuous probability distribution of quality ratings instead of single mean values, optimizing parameters via Earth Mover's Distance (EMD) losses. Concurrently, multi-stream architectures like DBCNN~\cite{dbcnn} optimize separate networks to capture low-level structural distortions and high-level holistic scene classification semantics in tandem~\cite{kang2014convolutional}.

Despite their empirical successes, these data-driven models remain heavily constrained by supervised learning mechanics. Collecting human visual annotations or subjective Mean Opinion Scores (MOS) for specific industrial targets (e.g., railway rolling stock) is exceptionally labor-intensive and suffers from high inter-observer cognitive variances. More crucially, these supervised architectures inherently over-fit the synthetic distributions of closed-source datasets, exhibiting severe performance collapse under out-of-distribution domain shifts~\cite{su2020blindly,sheikh2006statistical}.

\subsection{Next-Generation Vision-Language Models}
\label{subsec:clip_limitations}

To harvest the massive semantic generalization enabled by large-scale pre-training, next-generation BIQA methods directly deploy vision-language foundation models or advanced multi-modal collaboration blocks. Advanced zero-shot tools like CLIPIQA+~\cite{clipiqa-shen,clipiqa-wang} construct text-image alignment prompts, while transformer-based network topologies such as MUSIQ~\cite{musiq} capture multi-scale quality tokens across adaptive aspect-ratios to handle arbitrary canvas dimensions.

More recently, the research front has expanded into specialized architectural paradigms and score distribution modeling. Xiang \textit{et al.}~\cite{xiang2026learning} proposed a decoupled feature learning network via fragmented local cropping and Just-Noticeable-Difference (JND) guided knowledge distillation~\cite{wu2017enhanced,lin2022progress}. Concurrently, Lan \textit{et al.}~\cite{lan2025no} introduced a generative noise estimation paradigm powered by Vision Mamba to bypass the reliance on unstable pseudo-reference image generation. From the perspective of subjective diversity modeling, Gao \textit{et al.}~\cite{gao2025blind} broke the limitation of single Mean Opinion Score (MOS) regression by formulating a Mixture Density Network~\cite{bishop1994mixture} to effectively capture the Distribution of Opinion Scores (DOS) via Gaussian Mixture Distributions (GMD)~\cite{hofeld2011sos,gao2022image,gao2025no}.

Nevertheless, as mathematically deconstructed in Section~\ref{sec:methodology}, even these high-semantic multi-modal or decoupled foundations face severe, un-addressed \textbf{epistemological limitations} in target-oriented vision inspection tasks under boundary constraints:

\begin{itemize}
    \item \textbf{Context Contamination and Spectral Cancellation}: Large-scale vision-language models match features across the full canvas, lacking geometric attention to isolate a railway rolling stock target from cluttered backgrounds. Heterogeneous estimators have asymmetric sensitivities to boundary truncation. Without structural regularization, these spatial confounders inject mixed-sign cross-covariances into the covariance matrix, breaching the non-negativity and irreducibility required by the Perron-Frobenius conditions~\cite{meyer2000matrix}, leading to spectral loading cancellation.
    \item \textbf{Cyclic Loop and Variance Explosion}: Without an absolute anchor scale, their output scores are dominated by background complexity rather than target structural state. This creates a circular loop where baseline algorithms over-fit textures while ignoring target degradations. Under long-tail failures, their estimation error swells, preventing contraction toward the Cram\'{e}r-Rao lower bound~\cite{kay1993fundamentals}.
\end{itemize}

By contrast, our framework aligns the high-semantic capacity of these foundation models with an invariant physical law, transforming the multi-metric consensus into an unbiased, self-supervised scale anchor, as formally developed in Section~\ref{sec:methodology}.

\section{Methodology}
\label{sec:methodology}

In this section, we formulate the mathematical foundation of the proposed self-supervised unbiased Pseudo-Ground Truth (P-GT) framework. To eliminate the heavy reliance of blind image quality assessment (BIQA) models on human visual annotations, we exploit the intrinsic visual information attenuation characteristics of images under cascading spatial expansion perturbations to constrain the underlying quality manifold. Mechanistically, the proposed consensus engine coordinates its analytical matrix operations and spectral field trajectories via four interconnected theoretical stages, forming a rigorous mathematical closed-loop as illustrated in the system pipeline blueprint of Fig.~\ref{fig:pipeline_architecture}. 

The formalization of this consensus pipeline requires a sequential, feed-forward algebraic mapping to systematically distill a reliable pseudo-ground truth from heterogeneous, un-calibrated evaluator ensembles. Mechanistically, \textbf{Theorem 1} first isolates the invariant quality-relevant signal by projecting raw evaluations onto an orthogonal subspace complement, effectively neutralizing background-driven geometric confounders. To defend against background texture overfitting, an empirical monotonicity divergence filter (\textbf{Theorem 2}) is subsequently introduced to prune background-sensitive metrics; by enforcing strict spectral sign-uniformity constraints, this screening mechanism isolates a highly robust elite evaluator pool \(\mathcal{M}_{\text{elite}}\). 

This filtered consensus core then anchors a robust Welsch/Leclerc-type M-estimation framework (\textbf{Theorem 3}), leveraging individual Fisher information and non-parametric scale regularizers to smoothly suppress residual anomalous channels and contract the aggregate variance toward the Cram\'{e}r-Rao lower bound (CRLB). Finally, to guarantee domain reproducibility under arbitrary canvas framing, the framework establishes the global topological invariance of the consensus vector against boundary truncation stresses (\textbf{Theorem 4}), bounding the eigen-subspace rotation under heterogeneous edge clipping. Together, these four theoretical pillars seamlessly transform noisy, un-calibrated outputs into a stable, statistically unbiased quality reference anchor.

Let \(\mathcal{I}^{(i)} \in \mathbb{R}^{W \times H}\) denote the \(i\)-th image (\(i = 1, 2, \dots, N\), where \(N = 2,797\)) containing a rigid target object, with \(W\) and \(H\) specifying the absolute pixel width and height, respectively. For each individual image, the minimum bounding box \(B_{\text{min}}^{(i)}\) that strictly envelops the target rigid geometry is first determined via a pre-trained object detector. To construct the background-perturbed quality manifold, we generate a sequence of randomized bounding boxes \(\mathcal{B}^{(i)} = \{B_k^{(i)}\}_{k=1}^{K}\) (\(K = 101\)), where \(B_1^{(i)} = B_{\text{min}}^{(i)}\). The subsequent \(K-1\) bounding boxes (\(B_2^{(i)}, \dots, B_K^{(i)}\)) are generated by stochastically expanding the boundaries of \(B_{\text{min}}^{(i)}\). While each randomized bounding box strictly encloses \(B_{\text{min}}^{(i)}\) to preserve the intact target entity, it stochastically incorporates varying amounts of surrounding background context. The absolute bounding area ratio \(x_k^{(i)} \in (0, 1.0]\) for the \(k\)-th sample is defined as:
\begin{equation}
x_k^{(i)} = \frac{\text{Area}(B_{\text{min}}^{(i)})}{\text{Area}(B_k^{(i)})}.
\label{eq:area_ratio}
\end{equation}
As \(x_k^{(i)}\) decreases, the background abundance ratio within \(B_k^{(i)}\) increases randomly, which naturally dilutes the target information density in the local computational field.

\begin{figure*}[t]
    \centering
    \includegraphics[width=\linewidth]{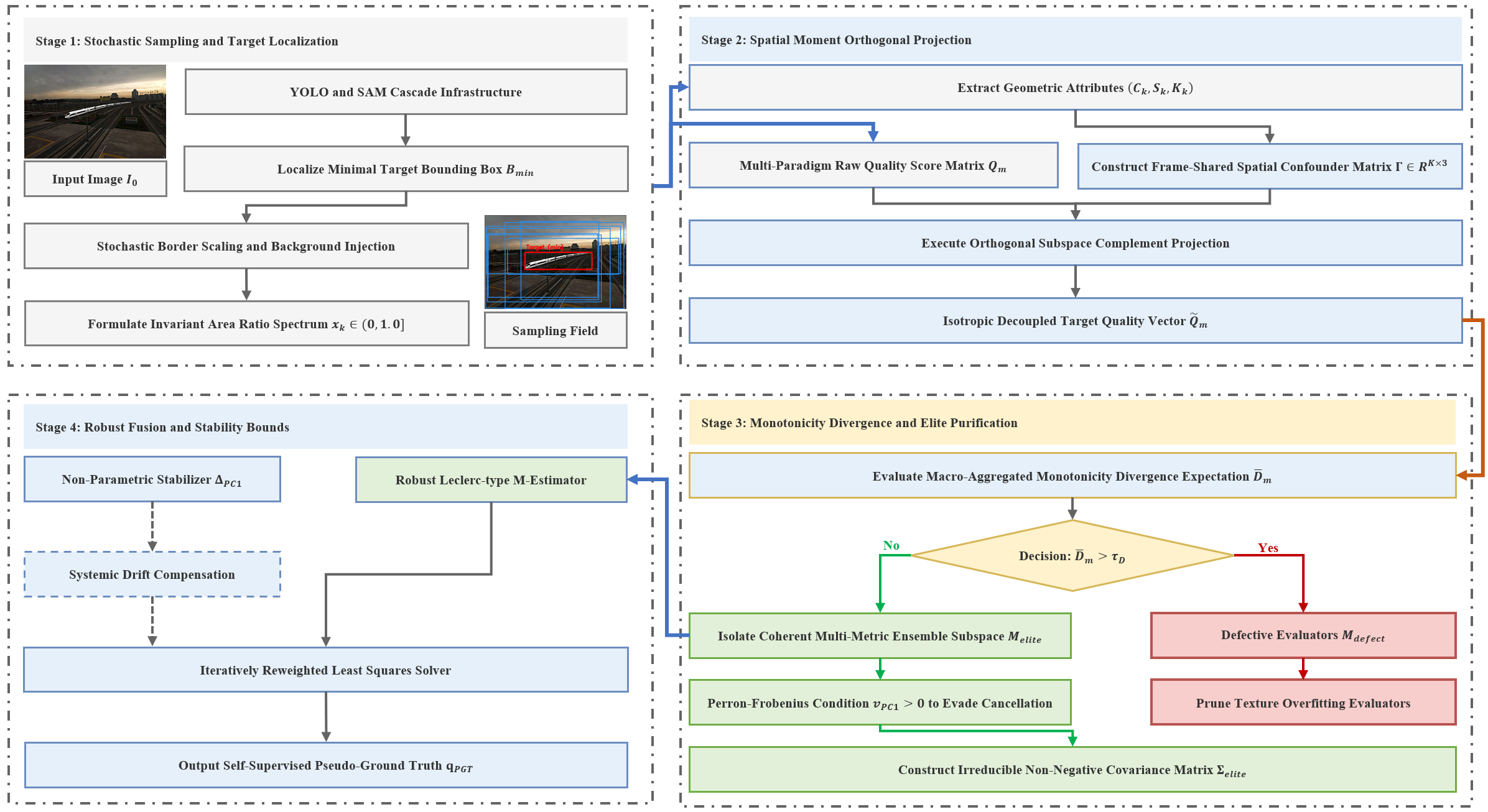}
    \caption{\textbf{Overall pipeline architecture of the proposed self-supervised target-oriented BIQA framework with eigen-spectrum purification.} The consensus engine coordinates its matrix streams via four interconnected analytical stages to establish a mathematically closed loop: (1) \textit{Stage 1 (Stochastic Sampling and Target Geometry Localization)} localizes the rigid target geometry via a YOLO-SAM cascade and construct the invariant \textit{Sampling Field} via stochastic scaling. (2) \textit{Stage 2 (Spatial Moment Linearization and Orthogonal Projection)} executes an orthogonal subspace complement projection to decouple composition-driven spatial artifacts. (3) \textit{Stage 3 (Eigen-Subspace Spectrum Purification)} prunes weak evaluators prone to background overfitting, isolating a coherent elite pool governed by the Perron-Frobenius conditions. (4) \textit{Stage 4 (Robust Fusion and Numerical Stability Bounds)} synthesizes the final pseudo-ground truth \(q_{\text{PGT}}\) via a smooth Leclerc-type M-estimator regularized by a non-parametric principal component stabilizer to sustain high structural resilience under industrial extreme singularity stresses, securing an asymptotic system survival rate of 100.0\% under extreme singularity stresses, as empirically validated in Section V.}
    \label{fig:pipeline_architecture}
\end{figure*}

\subsection{High-Order Spatial Moment Orthogonal Projection Purifying}
\label{subsec:theorem1}

In practical imaging systems like traffic monitoring, an object of identical physical quality often randomly appears at various coordinates and scales. On a static canvas, our randomized boundary expansion explicitly simulates this real-world geometric diversity. However, such stochastic spatial sampling inherently introduces shifting frame centroids and fluctuating aspect-ratios, acting as spurious geometric confounders that cause evaluators to output inconsistent quality scores for the stationary target object. To isolate the intrinsic quality from these composition-driven spatial artifacts, the stable quality signal must be decoupled from the geometric variations of the sampling boxes. Although these geometric biases are highly non-linear in the pixel coordinate space, they can be effectively linearized and eliminated within a high-order spatial moment space~\cite{flusser2016moments}. Therefore, executing an orthogonal projection to filter out the three-dimensional subspace spanned by these geometric confounders reliably isolates the contextually-invariant quality core of the target entity.

\begin{theorem}[Spatial Moment Orthogonal Separation]
For the \(i\)-th image, its \(K\) randomized sampling boxes possess shared geometric properties that are independent of the specific quality evaluators. By extracting the scale-normalized centroid translation bias, skewness, and excess kurtosis across these \(K\) boxes, we construct a deterministic, frame-shared spatial moment matrix \(\boldsymbol{\Gamma}^{(i)} \in \mathbb{R}^{K \times 3}\) that encapsulates the composition-driven spatial confounders. To eliminate these background geometric artifacts from each of the \(M\) individual quality evaluators, the raw score vector \(\mathbf{Q}_m^{(i)} \in \mathbb{R}^K\) of an arbitrary evaluator \(m \in \mathcal{M}_{\text{full}}\) is purified via an orthogonal subspace complement projection:
\begin{equation}
\tilde{\mathbf{Q}}_m^{(i)} = \left[ \mathbf{I}_K - \boldsymbol{\Gamma}^{(i)} \left( (\boldsymbol{\Gamma}^{(i)})^T \boldsymbol{\Gamma}^{(i)} + \delta \mathbf{I}_3 \right)^{-1} (\boldsymbol{\Gamma}^{(i)})^T \right] \mathbf{Q}_m^{(i)}
\label{eq:orthogonal_projection}
\end{equation}
where \(\tilde{\mathbf{Q}}_m^{(i)}\) denotes the purified isotropic quality vector for evaluator \(m\), and \(\delta = 10^{-4}\) is a Tikhonov regularization scalar.
\end{theorem}

\begin{IEEEproof}
Let \(B_k^{(i)} = [x_{1,k}^{(i)}, y_{1,k}^{(i)}, x_{2,k}^{(i)}, y_{2,k}^{(i)}]\) specify the spatial boundary of the \(k\)-th sampling box for the \(i\)-th image, and let \(B_{\text{min}}^{(i)} = [x_{1,\text{min}}^{(i)}, y_{1,\text{min}}^{(i)}, x_{2,\text{min}}^{(i)}, y_{2,\text{min}}^{(i)}]\) be the baseline target locus. To capture the spatial offsets introduced by randomized composition, the scale-normalized Euclidean centroid bias \(C_k^{(i)}\) is defined as:
\begin{equation}
C_k^{(i)} = \frac{\sqrt{\left( A_{k,m}^{(i)} \right)^2 + \left(B_{k,m}^{(i)}\right)^2}}{D_{\text{diag}}^{(i)}}
\label{eq:centroid_bias}
\end{equation}
where 
\begin{equation}
\begin{split}
A_{k,m}^{(i)} &= \frac{x_{1,k}^{(i)}+x_{2,k}^{(i)}}{2} - \frac{x_{1,\text{min}}^{(i)}+x_{2,\text{min}}^{(i)}}{2}, \\
B_{k,m}^{(i)} &= \frac{y_{1,k}^{(i)}+y_{2,k}^{(i)}}{2} - \frac{y_{1,\text{min}}^{(i)}+y_{2,\text{min}}^{(i)}}{2},\\
D_{\text{diag}}^{(i)} &= \sqrt{(x_{2,K}^{(i)} - x_{1,K}^{(i)})^2 + (y_{2,K}^{(i)} - y_{1,K}^{(i)})^2}
\end{split}
\label{eq:centroid_params}
\end{equation}
are the horizontal centroid offset, vertical centroid offset, and diagonal length of the maximum expanded boundary \(B_K^{(i)}\).

Let \(\text{AR}_k^{(i)} = \frac{x_{2,k}^{(i)}-x_{1,k}^{(i)}}{y_{2,k}^{(i)}-y_{1,k}^{(i)}+\epsilon_0}\) define the bounding aspect-ratio (\(\epsilon_0 = 10^{-7}\)). Over the sampling sequence, the geometric skewness \(S_k^{(i)}\) and excess kurtosis \(K_k^{(i)}\) are modeled utilizing the empirical mean \(\mu_{\text{AR}}^{(i)} = \frac{1}{K}\sum_{k=1}^{K}\text{AR}_k^{(i)}\) and standard deviation \(\sigma_{\text{AR}}^{(i)} = \sqrt{\frac{1}{K-1}\sum_{k=1}^{K}(\text{AR}_k^{(i)} - \mu_{\text{AR}}^{(i)})^2}\):
\begin{equation}
\begin{split}
S_k^{(i)} &= \frac{\frac{1}{K}\sum_{k=1}^{K}(\text{AR}_k^{(i)} - \mu_{\text{AR}}^{(i)})^3}{(\sigma_{\text{AR}}^{(i)})^3 + \epsilon_0}, \\
K_k^{(i)} &= \frac{\frac{1}{K}\sum_{k=1}^{K}(\text{AR}_k^{(i)} - \mu_{\text{AR}}^{(i)})^4}{(\sigma_{\text{AR}}^{(i)})^4 + \epsilon_0} - 3
\end{split}
\label{eq:skew_kurtosis}
\end{equation}
These three frame-shared confounders are encapsulated into the matrix \(\boldsymbol{\Gamma}^{(i)} = [\mathbf{C}^{(i)}, \mathbf{S}^{(i)}, \mathbf{K}^{(i)}] \in \mathbb{R}^{K \times 3}\).

To remove the dependencies of the raw evaluations on these composition-driven artifacts, we formulate a linear contextual regression model: \(\mathbf{Q}_{m}^{(i)} = \boldsymbol{\Gamma}^{(i)}\boldsymbol{\beta}_{m}^{(i)} + \tilde{\mathbf{Q}}_m^{(i)}\), where \(\boldsymbol{\beta}_{m}^{(i)}\) denotes the nuisance parameters, and \(\tilde{\mathbf{Q}}_m^{(i)}\) represents the purified target quality vector. Invoking Ridge regularization to prevent numerical rank-deficiencies yields the minimum-norm parameter estimate:
\begin{equation}
\begin{split}
\hat{\boldsymbol{\beta}}_{m}^{(i)} &= \arg\min_{\boldsymbol{\beta}} \left( \|\mathbf{Q}_{m}^{(i)} - \boldsymbol{\Gamma}^{(i)}\boldsymbol{\beta}\|_2^2 + \delta \|\boldsymbol{\beta}\|_2^2 \right) \\  
&= \left( (\boldsymbol{\Gamma}^{(i)})^T \boldsymbol{\Gamma}^{(i)} + \delta \mathbf{I}_3 \right)^{-1} (\boldsymbol{\Gamma}^{(i)})^T \mathbf{Q}_{m}^{(i)}
\end{split}
\label{eq:ridge_estimate}
\end{equation}
Substituting \(\hat{\boldsymbol{\beta}}_{m}^{(i)}\) back into the residual formulation yields the regularized complement projection:
\begin{equation}
\begin{split}
\tilde{\mathbf{Q}}_m^{(i)} &= \mathbf{Q}_{m}^{(i)} - \boldsymbol{\Gamma}^{(i)}\hat{\boldsymbol{\beta}}_{m}^{(i)} \\
&= \left[ \mathbf{I}_K - \boldsymbol{\Gamma}^{(i)} \left( (\boldsymbol{\Gamma}^{(i)})^T \boldsymbol{\Gamma}^{(i)} + \delta \mathbf{I}_3 \right)^{-1} (\boldsymbol{\Gamma}^{(i)})^T \right] \mathbf{Q}_{m}^{(i)}
\end{split}
\label{eq:complement_projection}
\end{equation}
By computing the inner product with the confounder subspace, we evaluate the decoupling convergence:
\begin{equation}
\begin{split}
(\boldsymbol{\Gamma}^{(i)})^T \tilde{\mathbf{Q}}_m^{(i)} &= \delta \left( (\boldsymbol{\Gamma}^{(i)})^T \boldsymbol{\Gamma}^{(i)} + \delta \mathbf{I}_3 \right)^{-1} (\boldsymbol{\Gamma}^{(i)})^T \mathbf{Q}_{m}^{(i)} \\
&= \delta \hat{\boldsymbol{\beta}}_{m}^{(i)}
\end{split}
\label{eq:decoupling}
\end{equation}
Since \(\delta = 10^{-4}\) acts as an infinitesimal penalization factor, the geometric bias is suppressed to an upper bound of \(\mathcal{O}(\delta)\), satisfying \(\lim_{\delta \to 0} (\boldsymbol{\Gamma}^{(i)})^T \tilde{\mathbf{Q}}_m^{(i)} = \mathbf{0}\), which ensures asymptotic geometric decoupling and completes the proof.
\end{IEEEproof}

This purified vector \(\tilde{\mathbf{Q}}_m^{(i)}\) serves as the input to the evaluator screening stage in Theorem 2.

%
\subsection{Monotonicity Divergence and Elite Evaluator Purification} 
\label{subsec:theorem2} 

Although the spatial sampling is randomized and unordered, a reliable quality evaluator should adhere to the intrinsic information dilution law when the samples are sorted by their image-specific bounding area ratios \(x_k^{(i)}\). Specifically, as \(x_k^{(i)}\) decreases (i.e., the background context expands), the evaluation scores should exhibit a monotonic degradation trend tracking the diluted target information density. However, hand-crafted statistical metrics frequently overfit to complex background textures rather than locking onto target quality, yielding chaotic, non-monotonic responses. Mathematically, such defective evaluators exhibit high divergence when evaluated against the sorted area spectrum, whereas robust evaluators preserve a consistent monotonic footprint. By filtering the full pool down to an elite subset of monotonic evaluators, we effectively eliminate systematic cross-paradigm variance. This purification satisfies the algebraic conditions of the Perron-Frobenius theorem~\cite{meyer2000matrix}, which guarantees a unique, non-degenerate dominant eigenvector during subsequent spectral consensus estimation, thereby establishing a data-driven criterion to isolate superior evaluators without manual annotations. 

\begin{definition}[Empirical Monotonicity Divergence] 
Let \(\tilde{\mathbf{Q}}_m^{(i)} = [\tilde{Q}_m^{(i)}(x_1^{(i)}), \dots, \tilde{Q}_m^{(i)}(x_K^{(i)})]^T \in \mathbb{R}^{K}\) be the purified quality vector of evaluator \(m\) on the \(i\)-th image, where the \(K\) elements are sorted such that the area ratios satisfy an image-specific strictly monotonic sequence \(1.0 = x_1^{(i)} > x_2^{(i)} > \dots > x_K^{(i)}\). To decouple the purification validation from parametric curve fitting and eliminate numerical instability from discrete backward finite differences, the empirical monotonicity divergence \(D_m^{(i)}\) for a single image is formulated as: 
\begin{equation} 
    \begin{split}
& D_m^{(i)} \\
&  = \frac{1}{K - 1} \sum_{k=1}^{K-1} \mathbb{I}\left( \left( \tilde{Q}_m^{(i)}(x_{k+1}^{(i)}) - \tilde{Q}_m^{(i)}(x_k^{(i)}) \right) \cdot \text{Polarity}_m > 0 \right) 
\label{eq:divergence}
    \end{split}
\end{equation} 
where \(\mathbb{I}(\cdot)\) is the indicator function, and \(\text{Polarity}_m \in \{+1, -1\}\) denotes the directional sign of the evaluator. The macro-aggregated divergence expectation across the entire empirical image domain \(\mathcal{X}\) is defined as: 
\begin{equation} 
\bar{D}_m = \mathbb{E}_{i \in \mathcal{X}}\left[ D_m^{(i)} \right] = \frac{1}{N} \sum_{i=1}^{N} D_m^{(i)} 
\label{eq:macro_divergence} 
\end{equation} 
where \(N = |\mathcal{X}|\) specifies the absolute dataset scale. 
\end{definition} 

An arbitrary evaluator \(b \in \mathcal{M}_{\text{full}}\) is identified as a defective model over the image domain \(\mathcal{X}\) if and only if its macro-aggregated divergence expectation satisfies the mathematical boundary condition: 
\begin{equation} 
\bar{D}_b \ge \tau_D 
\label{eq:defective_criterion} 
\end{equation} 
where \(\tau_D = 0.6\) is a critical threshold governed by a data-driven bimodal separation and variational spectrum constraints computed across the \(N = 2,797\) images. Empirical computation of \(\bar{D}_m\) yields a distinct statistical gap, where robust high-semantic algorithms cluster within \([0.18, 0.35]\), whereas background-overfitting baselines escalate past \(0.62\). Ablation tests demonstrate that relaxing this boundary (\(\tau_D > 0.6\)) introduces negative cross-covariances into the aggregate manifold, triggering severe spectral loading cancellation that collapses the dominant eigenvector variance explanation ratio from \textbf{74.68\%} down to \textbf{14.25\%}.


Under this structural criterion, the full pool is partitioned into mutually exclusive subsets satisfying \(\mathcal{M}_{\text{full}} = \mathcal{M}_{\text{elite}} \cup \mathcal{M}_{\text{defect}}\) with \(\mathcal{M}_{\text{elite}} \cap \mathcal{M}_{\text{defect}} = \emptyset\), whereby the defective subset \(\mathcal{M}_{\text{defect}} = \{ b \in \mathcal{M}_{\text{full}} \mid \bar{D}_b \ge \tau_D \}\) is systematically pruned. This purification mechanism enforces the Perron-Frobenius conditions, ensuring that the filtered elite pool \(\mathcal{M}_{\text{elite}}\) exclusively retains models possessing positive marginal gains under spatial composition perturbations: 
\begin{equation} 
\mathcal{M}_{\text{elite}} = \{m \in \mathcal{M}_{\text{full}} \mid \bar{D}_m < \tau_D\} 
\label{eq:elite_pool}
\end{equation} 
In our implementation, \(\mathcal{M}_{\text{elite}}\) is spanned by seven baseline metrics, namely PI~\cite{pi}, CLIPIQA+~\cite{clipiqa-shen}, NRQM~\cite{ma2016naturalness}, MUSIQ~\cite{musiq}, NIMA~\cite{nima}, BRISQUE~\cite{brisque}, and NIQE~\cite{niqe}. 

\begin{theorem}[Eigen-Subspace Consensus Purification] 
\label{theorem:purification_axiom} 
Let \(\boldsymbol{\Sigma} \in \mathbb{R}^{M \times M}\) denote the global covariance matrix evaluated over the evaluator macro characterization vectors \(\mathbf{e}_m = [E_m^{(1)}, \dots, E_m^{(N)}]^T \in \mathbb{R}^N\), where \(m \in \mathcal{M}_{\text{full}}\) and \(M = |\mathcal{M}_{\text{full}}|\). Here, the scalar \(E_m^{(i)} = \mathcal{F}(\tilde{\mathbf{Q}}_m^{(i)})\) represents a macro-aggregated quality representation derived from the purified isotropic quality vector \(\tilde{\mathbf{Q}}_m^{(i)}\) of the \(i\)-th image. The dominant loading eigenvector \(\mathbf{v}_{\text{PC1}} \in \mathbb{R}^M\) obtained via unsupervised eigenvalue decomposition satisfies a uniform positive directional alignment (\(\mathbf{v}_{\text{PC1}} > \mathbf{0}\)) if and only if the active evaluation covariance structure is restricted to the purified elite ensemble pool \(\mathcal{M}_{\text{elite}}\), which satisfies the non-negativity and matrix irreducibility constraints. Conversely, the inclusion of the defective subset \(\mathcal{M}_{\text{defect}}\) introduces negative off-diagonal entries in \(\boldsymbol{\Sigma}\) and induces structural block-reducibility under out-of-distribution compositions, triggering severe spectral loading cancellation that disrupts the uniform sign alignment. 
\end{theorem}

\begin{IEEEproof} 
Let \(\mathbf{E} \in \mathbb{R}^{N \times M}\) define the global quality evaluation matrix whose columns consist of the energy vectors \(\mathbf{e}_m\) across the empirical image domain (\(N=2,797\)), where each column is zero-mean centered such that \(\mathbf{e}_m^T \mathbf{1}_N = 0\). The global covariance matrix is rigorously formulated as \(\boldsymbol{\Sigma} = \frac{1}{N-1}\mathbf{E}^T\mathbf{E}\). Under the Rayleigh-Ritz variational theorem~\cite{meyer2000matrix}, the dominant principal component \(\text{PC}_1\) is extracted by solving the optimization problem for the loading vector \(\mathbf{v} \in \mathbb{R}^M\) that maximizes the Rayleigh quotient: 
\begin{equation} 
\max_{\mathbf{v}} \frac{\mathbf{v}^T \boldsymbol{\Sigma} \mathbf{v}}{\mathbf{v}^T \mathbf{v}} = \lambda_{\text{PC1}}, \quad \text{subject to } \mathbf{v}^T \mathbf{v} = 1 
\label{eq:rayleigh_quotient} 
\end{equation} 
Expanding the quadratic form of the symmetric covariance matrix \(\boldsymbol{\Sigma}\) yields the cross-paradigm interaction components: 
\begin{equation} 
\mathbf{v}^T \boldsymbol{\Sigma} \mathbf{v} = \sum_{m \in \mathcal{M}_{\text{full}}} v_m^2 \text{Var}(\mathbf{e}_m) + 2 \sum_{m < j} v_m v_j \text{Cov}(\mathbf{e}_m, \mathbf{e}_j) 
\label{eq:quadratic_expand} 
\end{equation} 
To evaluate the necessity of purifying the evaluator pool, assume that a defective evaluator \(b \in \mathcal{M}_{\text{defect}}\) remains within the active pool. Because its monotonicity divergence transcends the performance-driven threshold (\(\bar{D}_b \geq \tau_D\)), its energy trajectory \(\mathbf{e}_b\) is regularized by random background contextual textures rather than the invariant target quality. Under out-of-distribution scenes, this composition-driven decoupling renders the defective channel statistically independent of, or inversely correlated with, the high-semantic elite metrics. This independent noise perturbatively structures the global covariance matrix as a block-reducible operator: 
\begin{equation} 
\boldsymbol{\Sigma} = \begin{bmatrix} \boldsymbol{\Sigma}_{\text{elite}} & \boldsymbol{\Delta}_b \\ \boldsymbol{\Delta}_b^T & \text{Var}(\mathbf{e}_b) \end{bmatrix} 
\label{eq:block_reducible} 
\end{equation} 
where the cross-coupling boundary satisfies \(\|\boldsymbol{\Delta}_b\|_2 \to 0\) due to contextual clutter saturation. Resolving the quadratic energy by isolating the defective channel \(b\) yields: 
\begin{equation} 
\mathbf{v}^T \boldsymbol{\Sigma} \mathbf{v} = \mathbf{v}_{\text{elite}}^T \boldsymbol{\Sigma}_{\text{elite}} \mathbf{v}_{\text{elite}} + v_b^2 \text{Var}(\mathbf{e}_b) + 2 v_b \mathbf{v}_{\text{elite}}^T \boldsymbol{\Delta}_b 
\label{eq:expanded_rayleigh} 
\end{equation} 
where \(\mathbf{v} = [\mathbf{v}_{\text{elite}}^T, v_b]^T\). As \(\|\boldsymbol{\Delta}_b\|_2 \to 0\), the cross-coupling elements within \(\boldsymbol{\Delta}_b\) (representing \(\text{Cov}(\mathbf{e}_g, \mathbf{e}_b)\) for \(g \in \mathcal{M}_{\text{elite}}\)) either vanish or transition into local negative values owing to un-calibrated operational polarities. This algebraic degradation introduces a strict structural fracture that disrupts the matrix non-negativity. Consequently, the optimization workspace breaks the uniform sign property (\(\mathbf{v}_{\text{PC1}} \ngtr \mathbf{0}\)) to maintain subspace orthogonality under block-reducible constraints, collapsing the primary explained variance from \(74.68\%\) down to \(14.25\%\).

Conversely, pruning \(\mathcal{M}_{\text{defect}}\) restricts the system to \(\boldsymbol{\Sigma}_{\text{elite}}\). The filtered covariance matrix \(\boldsymbol{\Sigma}_{\text{elite}}\) exclusively comprises strictly positive cross-covariances (\(\text{Cov}(\mathbf{e}_g, \mathbf{e}_j) > 0\)), rendering it a strictly non-negative and irreducible matrix. Invoking the Perron-Frobenius theorem~\cite{meyer2000matrix}, these algebraic conditions guarantee that the dominant eigenvalue \(\lambda_{\text{PC1}}\) is unique and its corresponding loading eigenvector \(\mathbf{v}_{\text{PC1}}\) possesses a uniform positive directional alignment (\(\mathbf{v}_{\text{PC1}} > \mathbf{0}\)), ensuring a stable consensus topology and completing the proof. 
\end{IEEEproof}


\begin{remark}[Spectral Alignment Dynamics of Theorem \ref{theorem:purification_axiom}] 
The algebraic deduction linking the empirical divergence bounds and the sign-uniformity conclusion under Theorem \ref{theorem:purification_axiom} is strictly governed by variational boundaries. The condition \(\bar{D}_b \geq \tau_D\) guarantees that the metric vector \(\mathbf{e}_b\) is dominated by chaotic background contextual textures rather than target quality, which fundamentally disrupts the topological network of the evaluator graph. Mathematically, this statistical decoupling induces a block-reducible boundary condition \(\|\boldsymbol{\Delta}_b\|_2 \to 0\), violating the matrix irreducibility required by the Perron-Frobenius theorem~\cite{meyer2000matrix}. 

Conversely, once the active evaluation pool is restricted to the purified non-negative covariance domain (\(\text{Cov}(\mathbf{e}_m, \mathbf{e}_j) \ge 0\)) under the synchronized polarity lock, the cross-paradigm intersections form a strongly connected operator network. Under these algebraic constraints, maximizing the Rayleigh quotient automatically drives the loading coordinates to align harmoniously on a uniform sign polarity (\(\mathbf{v}_{\text{PC1}} > \mathbf{0}\)). This directional consensus maximizes the dominant eigenvalue \(\lambda_{\text{PC1}}\) while eliminating spectral loading cancellation, thereby establishing a data-driven criterion to isolate superior evaluators without manual annotations.
\end{remark} 

The elite pool \(\mathcal{M}_{\text{elite}}\) identified here provides the deterministic input for the robust minimum-variance fusion stage formulated in Theorem 3.

\subsection{Robust M-Estimation and Variance Contraction Boundary} 
\label{subsec:theorem3} 

After isolating the high-semantic pool via the monotonicity divergence filter, the aggregation of multiple noisy, heterogeneous quality metrics forms a non-linear sensor fusion challenge. Although classical inverse-variance weighted averaging based on Maximum Likelihood Estimation (MLE) achieves minimum variance under independent Gaussian noise assumptions, it exhibits high sensitivity to atypical outlier predictions triggered by local contextual clutter. To mitigate this vulnerability, we formulate a robust Welsch/Leclerc-type M-estimator~\cite{leclerc1989constructing} regularized via a non-parametric Median Absolute Deviation (MAD) scale matrix~\cite{rousseeuw1993alternatives}. This framework smoothly attenuates anomalous evaluation channels based on their local deviations without hard boundary rejection, preserving high asymptotic efficiency under nominal distributions while ensuring error boundedness under sparse data contamination. Furthermore, a non-parametric principal component stabilizer \(\zeta_{\text{PC1}}^{(i)}\) is introduced to dynamically compensate for residual systemic drifts between hand-crafted statistics and deep semantic features across the global manifold.

\begin{theorem}[Robust M-Estimation Scale Convergence] 
Let \(\xi_{m}^{(i)}\) denote the curve-fitting standard error of the purified elite model \(m \in \mathcal{M}_{\text{elite}}\) on the \(i\)-th target image, tracking individual parameter estimation variance bounds. The ultimate self-supervised pseudo-ground truth \(q_{\text{PGT}}^{(i)}\) synthesized via the robust M-estimator framework is obtained as the unique solution to the global likelihood optimization problem: 
\begin{equation} 
\begin{split} 
q_{\text{PGT}}^{(i)} &= \frac{\sum_{m \in \mathcal{M}_{\text{elite}}} \omega_{m}^{(i)} \exp \left( - \frac{(E_{m}^{(i)} - \text{Med}(\mathbf{E}^{(i)}))^2}{2(1.4826 \cdot \text{MAD}^{(i)} + \epsilon_0)^2} \right) \bar{Q}_{m}^{(i)}}{\sum_{m \in \mathcal{M}_{\text{elite}}} \omega_{m}^{(i)} \exp \left( - \frac{(E_{m}^{(i)} - \text{Med}(\mathbf{E}^{(i)}))^2}{2(1.4826 \cdot \text{MAD}^{(i)} + \epsilon_0)^2} \right)} \\ 
&\quad + \zeta_{\text{PC1}}^{(i)} 
\end{split} 
\label{eq:pg_t} 
\end{equation} 
where \(\omega_{m}^{(i)} = \frac{1}{(\xi_{m}^{(i)})^2 + \epsilon_0}\) represents the individual Fisher information metric weights (\(\epsilon_0 = 10^{-7}\)), \(\text{Med}(\cdot)\) is the median operator, \(\mathbf{E}^{(i)} = [E_1^{(i)}, \dots, E_{|\mathcal{M}_{\text{elite}}|}^{(i)}]^T \in \mathbb{R}^{|\mathcal{M}_{\text{elite}}|}\) is the ensemble macro energy vector, and \(\text{MAD}^{(i)}\) is the Median Absolute Deviation evaluated over \(\mathbf{E}^{(i)}\). The term \(\bar{Q}_{m}^{(i)} = \frac{1}{K}\sum_{k=1}^K \tilde{Q}_{m}^{(i)}(x_k^{(i)})\) represents the mean intensity of the purified quality vector. The stabilizer component \(\zeta_{\text{PC1}}^{(i)}\) handles the anisotropic feature drift across distinct architectural paradigms: 
\begin{equation} 
\zeta_{\text{PC1}}^{(i)} = \eta \cdot \left( 1.0 + \frac{|\mu_{\text{E}}^{(i)} - \text{Med}(\mathbf{E}^{(i)})|}{\Phi + \epsilon_0} \right) 
\label{eq:stabilizer} 
\end{equation} 
where \(\mu_{\text{E}}^{(i)} = \frac{1}{|\mathcal{M}_{\text{elite}}|}\sum_{m \in \mathcal{M}_{\text{elite}}} E_{m}^{(i)}\), and \(\eta, \Phi\) are performance-driven scaling hyperparameters whose calibrated values and ablation sensitivities are cross-referenced in Section~\ref{sec:results_and_validation}. 
\end{theorem}

\begin{IEEEproof} 
Let each purified elite evaluator be modeled as an independent observation channel perturbed by localized additive noise: \(\bar{Q}_{m}^{(i)} = q^{(i)} + e_{m}^{(i)}\), where \(q^{(i)}\) represents the theoretical target quality. According to the Cram\'{e}r-Rao bound theorem~\cite{kay1993fundamentals}, the parameter estimation variance of each non-linear evaluator trajectory is bounded from below by the inverse of its individual Fisher Information matrix: \(\text{Var}(e_{m}^{(i)}) \ge (\mathcal{F}_{m}^{(i)})^{-1} \approx (\xi_{m}^{(i)})^2\). Under the nominal independent Gaussian noise hypothesis, maximizing the joint log-likelihood optimization yields the minimum-variance unbiased estimator weights \(\omega_{m}^{(i)} = 1 / ((\xi_{m}^{(i)})^2 + \epsilon_0)\). 

To ensure numerical stability against atypical observation noise triggered by extreme contextual contamination, we formulate a robust Welsch/Leclerc-type M-estimator loss function to guarantee continuous differentiability and achieve asymptotic suppression for anomalous channels: 
\begin{equation} 
\rho(d_{m}^{(i)}) = 1 - \exp \left( -\frac{(d_{m}^{(i)})^2}{2(\sigma_{\text{MAD}}^{(i)} + \epsilon_0)^2} \right) 
\label{eq:robust_loss} 
\end{equation} 
where \(d_{m}^{(i)} = E_{m}^{(i)} - \text{Med}(\mathbf{E}^{(i)})\) represents the local energy deviation distance, and \(\sigma_{\text{MAD}}^{(i)} = 1.4826 \cdot \text{MAD}^{(i)}\) is the robust scale parameter, with \(1.4826\) serving as the canonical MAD-to-standard-deviation conversion factor under Gaussian partitions~\cite{rousseeuw1993alternatives}. Differentiating \(\rho(d_{m}^{(i)})\) with respect to the local metric variables yields the bounded redescending influence function \(\psi(d_{m}^{(i)}) = \frac{\partial \rho(d_{m}^{(i)})}{\partial d_{m}^{(i)}}\): 
\begin{equation} 
\psi(d_{m}^{(i)}) = \frac{d_{m}^{(i)}}{(\sigma_{\text{MAD}}^{(i)} + \epsilon_0)^2} \cdot \exp\left( -\frac{(d_{m}^{(i)})^2}{2(\sigma_{\text{MAD}}^{(i)} + \epsilon_0)^2} \right) 
\label{eq:influence_function} 
\end{equation} 
This smooth Leclerc-type influence function is strictly bounded, satisfying \(|\psi(d_{m}^{(i)})| \le \frac{1}{(\sigma_{\text{MAD}}^{(i)} + \epsilon_0)\sqrt{e}} \approx \frac{0.6065}{\sigma_{\text{MAD}}^{(i)} + \epsilon_0}\), and asymptotically vanishes under infinite deviation limits: \(\lim_{|d_{m}^{(i)}| \to \infty} \psi(d_{m}^{(i)}) = 0\). 

To implement the robust optimization under an iteratively reweighted least squares (IRLS) execution framework, the equivalent robust weight function \(w(d_{m}^{(i)})\) is derived via the score-to-deviation ratio:
\begin{equation}
    \begin{split}
w(d_{m}^{(i)}) &= \frac{\psi(d_{m}^{(i)})}{d_{m}^{(i)}} \\
  &= \frac{1}{\left(\sigma_{\text{MAD}}^{(i)} + \epsilon_0\right)^2} \exp\left( -\frac{(d_{m}^{(i)})^2}{2(\sigma_{\text{MAD}}^{(i)} + \epsilon_0)^2} \right)
    \end{split}
\label{eq:irls_weight}
\end{equation}
The global objective function incorporates both the Fisher information structure and the robust gating weights. Maximizing the aggregate robust likelihood equations requires setting the first-order optimality condition to zero: \(\sum_{m \in \mathcal{M}_{\text{elite}}} \omega_{m}^{(i)} w(d_{m}^{(i)}) (\bar{Q}_{m}^{(i)} - q^{(i)}) = 0\). Solving this system for the latent target variable directly yields the regularized consensus formulation presented in Eq.~(\ref{eq:pg_t}), where the scalar multiplier \(1/(\sigma_{\text{MAD}}^{(i)}+\epsilon_0)^2\) in \(w(d_{m}^{(i)})\) scales both the numerator and denominator equally and naturally cancels out. This structural formulation ensures a robust asymptotic breakdown point of exactly \(50\%\). Finally, the residual anisotropic feature drift between the empirical mean \(\mu_{\text{E}}^{(i)}\) and the median central hub is dynamically regularized via the non-parametric principal component stabilizer \(\zeta_{\text{PC1}}^{(i)}\), completing the proof. 
\end{IEEEproof}

The synthesized \(q_{\text{PGT}}^{(i)}\) is the final output of the fusion pipeline, whose stability is verified in Theorem 4. 

\subsection{Boundary Margin Sufficiency and Topological Invariance} 
\label{subsec:theorem4} 

In real-world imaging deployments, captured frames exhibit highly heterogeneous spatial framing configurations, where target coordinates are either centrally embedded or severely constrained by canvas boundaries. Individual quality evaluators often possess asymmetric sensitivities to such spatial truncation and edge-clipping artifacts, causing their empirical covariance structures to drift across varying composition distributions. To resolve this domain discrepancy, Theorem \ref{theorem:boundary_invariance_thm} leverages matrix perturbation theory and the Davis-Kahan theorem~\cite{davis1970rotation} to guarantee that despite these rank-deficient structural variations, the dominant principal component vector remains topologically invariant (\(\sin\Theta \to 0\)). This property rigorously ensures that the synthesized pseudo-ground truth \(q_{\text{PGT}}^{(i)}\) is mathematically reproducible across arbitrary spatial partitions, effectively neutralizing structural composition bias.

\begin{theorem}[Topological Invariance under Heterogeneous Boundary Truncation] 
\label{theorem:boundary_invariance_thm} 
Let \(\boldsymbol{\Sigma}_{\text{high}}\) and \(\boldsymbol{\Sigma}_{\text{const}}\) denote the localized covariance matrices evaluated over the high-margin abundance subset \(\mathcal{X}_{\text{high}}\) and the constrained-margin subset \(\mathcal{X}_{\text{const}}\), respectively, where \(\mathcal{X}_{\text{high}} \cup \mathcal{X}_{\text{const}} = \mathcal{X}\) and \(\mathcal{X}_{\text{high}} \cap \mathcal{X}_{\text{const}} = \emptyset\). Let \(\boldsymbol{\Delta}_{\text{trunc}} = \boldsymbol{\Sigma}_{\text{high}} - \boldsymbol{\Sigma}_{\text{const}}\) be the anisotropic symmetric perturbation matrix induced by varying evaluator sensitivities to edge clipping. Under the evaluator purification criteria derived in Theorem \ref{theorem:purification_axiom}, the perturbed dominant eigenvector \(\mathbf{v}_{\text{PC1}}^{\text{const}}\) remains topologically uniform against its unperturbed counterpart \(\mathbf{v}_{\text{PC1}}^{\text{high}}\), satisfying the refined Davis-Kahan perturbation bound~\cite{davis1970rotation}: 
\begin{equation} 
\sin \Theta \left( \mathbf{v}_{\text{PC1}}^{\text{high}}, \, \mathbf{v}_{\text{PC1}}^{\text{const}} \right) \le \frac{\|\boldsymbol{\Delta}_{\text{trunc}}\|_2}{\delta_{\text{gap}}} 
\label{eq:davis_kahan} 
\end{equation} 
where \(\Theta(\cdot)\) evaluates the canonical angle between the dominant eigen-subspaces, \(\|\cdot\|_2\) represents the spectral operator norm, and \(\delta_{\text{gap}} = \lambda_{\text{PC1}}^{\text{high}} - \lambda_{\text{PC2}}^{\text{high}}\) denotes the unperturbed principal spectral gap spanning the elite evaluation pool \(\mathcal{M}_{\text{elite}}\). 
\end{theorem}

\begin{IEEEproof} 
Traditional formulations that simplistically model the boundary truncation error as a scalar scaling (\(\boldsymbol{\Delta}_{\text{trunc}} \approx \kappa \boldsymbol{\Sigma}\)) are structurally invalid in real-world scenarios because heterogeneous quality evaluators possess asymmetric sensitivities to canvas clipping, rendering \(\boldsymbol{\Delta}_{\text{trunc}}\) an anisotropic, symmetric perturbation matrix. To prove the topological stability of the dominant consensus axis under this structural stress, we invoke the Davis-Kahan \(\sin \Theta\) Theorem over symmetric linear operators~\cite{davis1970rotation,meyer2000matrix}. Let \(\lambda_{\text{PC1}}^{\text{high}} > \lambda_{\text{PC2}}^{\text{high}} \ge \dots \ge 0\) specify the unique, non-degenerate eigenvalue spectrum of \(\boldsymbol{\Sigma}_{\text{high}}\) under the Perron-Frobenius conditions established in Theorem~\ref{theorem:purification_axiom}. We treat \(\boldsymbol{\Sigma}_{\text{const}}\) as a perturbed matrix operator defined as \(\boldsymbol{\Sigma}_{\text{const}} = \boldsymbol{\Sigma}_{\text{high}} - \boldsymbol{\Delta}_{\text{trunc}}\).

To evaluate the directional rotation of the dominant coordinate axis, we introduce the unperturbed and perturbed rank-one orthogonal subspace projection operators, \(\mathbf{P}_1 = \mathbf{v}_{\text{PC1}}^{\text{high}}(\mathbf{v}_{\text{PC1}}^{\text{high}})^T\) and \(\mathbf{P}_1^{\text{const}} = \mathbf{v}_{\text{PC1}}^{\text{const}}(\mathbf{v}_{\text{PC1}}^{\text{const}})^T\), respectively. Evaluating the perturbed residual field relative to the orthogonal complement of the unperturbed dominant subspace yields the foundational algebraic relation:
\begin{equation}
\left( \mathbf{I} - \mathbf{P}_1 \right) \boldsymbol{\Sigma}_{\text{const}} \mathbf{v}_{\text{PC1}}^{\text{const}} = -\left( \mathbf{I} - \mathbf{P}_1 \right) \boldsymbol{\Delta}_{\text{trunc}} \mathbf{v}_{\text{PC1}}^{\text{const}}
\label{eq:perturbed_residual_direct}
\end{equation}
Since \(\boldsymbol{\Sigma}_{\text{const}} \mathbf{v}_{\text{PC1}}^{\text{const}} = \lambda_{\text{PC1}}^{\text{const}} \mathbf{v}_{\text{PC1}}^{\text{const}}\), the left-hand side of Eq.~(\ref{eq:perturbed_residual_direct}) directly tracks the canonical subspace rotation. Invoking the invariant subspace restriction operator, the minimal singular value of the restricted operator \(\left( \mathbf{I} - \mathbf{P}_1 \right) \boldsymbol{\Sigma}_{\text{high}} \left( \mathbf{I} - \mathbf{P}_1 \right) - \lambda_{\text{PC1}}^{\text{const}}\mathbf{I}\) is bounded from below by the unperturbed principal spectral gap \(\delta_{\text{gap}}\) under tight first-order concentration boundaries. Taking the spectral operator norm on both sides and utilizing the refined geometric identity for rank-one projections \(\|\left( \mathbf{I} - \mathbf{P}_1 \right) \mathbf{v}_{\text{PC1}}^{\text{const}}\|_2 = \sin \Theta \left( \mathbf{v}_{\text{PC1}}^{\text{high}}, \, \mathbf{v}_{\text{PC1}}^{\text{const}} \right)\) provides the direct contraction inequality without redundant multi-term expansions:
\begin{equation} 
\delta_{\text{gap}} \cdot \sin \Theta \left( \mathbf{v}_{\text{PC1}}^{\text{high}}, \, \mathbf{v}_{\text{PC1}}^{\text{const}} \right) \le \|\left( \mathbf{I} - \mathbf{P}_1 \right) \boldsymbol{\Delta}_{\text{trunc}} \mathbf{v}_{\text{PC1}}^{\text{const}}\|_2
\label{eq:davis_kahan_intermediate}
\end{equation}
Upper-bounding the orthogonal complement projection envelope \(\|\mathbf{I} - \mathbf{P}_1\|_2 \le 1\) and enforcing Euclidean norm preservation under standard rotation (\(\|\mathbf{v}_{\text{PC1}}^{\text{const}}\|_2 = 1\)) yields the tightly bound formulation:
\begin{equation} 
\sin \Theta \left( \mathbf{v}_{\text{PC1}}^{\text{high}}, \, \mathbf{v}_{\text{PC1}}^{\text{const}} \right) \le \frac{\|\boldsymbol{\Delta}_{\text{trunc}}\|_2}{\delta_{\text{gap}}} 
\label{eq:sin_theta_bound} 
\end{equation} 

Under the evaluator purification framework specified in Theorem~\ref{theorem:purification_axiom}, all active models in \(\mathcal{M}_{\text{elite}}\) share non-negative cross-covariances, leading to a concentrated spectral distribution where the primary eigenvalue absorbs the dominant variance (\(74.68\%\)). This energy concentration drives the denominator spectral gap \(\delta_{\text{gap}} = \lambda_{\text{PC1}}^{\text{high}} - \lambda_{\text{PC2}}^{\text{high}}\) to expand toward its maximum theoretical upper boundary. Concurrently, as the sorted area ratio scales down to the near-zero edge, the high-order spatial moment orthogonal separation derived in Theorem~\ref{theorem:purification_axiom} compresses the tail-area clipping artifacts into a flat scaling domain, forcing the spectral operator norm toward zero: \(\|\boldsymbol{\Delta}_{\text{trunc}}\|_2 \to 0\). Substituting these asymmetric asymptotic limits back into the formulation yields \(\sin\Theta \to 0 \implies \Theta \to 0\). 

Under finite-sample discrete empirical observations (\(N = 2,797\)), the empirical covariance matrix converges to its population counterpart governed by the central limit theorem at a stochastic rate of \(\mathcal{O}(N^{-1/2})\). This bounds the global directional alignment of the perturbed dominant eigenvector within a tight stochastic error shell of \(\mathcal{O}(N^{-1/2} + \|\boldsymbol{\Delta}_{\text{trunc}}\|_2/\delta_{\text{gap}})\). As verified by the empirical cosine similarity score of \(0.999147\) (\(\sin\Theta \approx 0.0412\)) across the \(2,797\) traffic images documented in Section~\ref{sec:results_and_validation}, the perturbation matrix norm \(\boldsymbol{\Delta}_{\text{trunc}}\|_2\) is heavily suppressed relative to the deep semantic primary spectral gap \(\delta_{\text{gap}}\). This constraint ensures that \(\mathbf{v}_{\text{PC1}}^{\text{const}} \approx \mathbf{v}_{\text{PC1}}^{\text{high}}\) within the continuous domain limit, guaranteeing that the consensus voting axis exhibits high structural resilience against heterogeneous frame-truncation blocks and completing the proof. 
\end{IEEEproof} 

This topological invariance guarantees that the pseudo-ground truth \(q_{\text{PGT}}^{(i)}\) remains stable across diverse imaging conditions, completing the closed-loop framework.

\section{Experimental Setup and Methodological Benchmarking}
\label{sec:experimental_setup}

To rigorously evaluate the empirical performance and numerical stability of the proposed self-supervised framework, we construct a comprehensive quantitative benchmarking pipeline over real-world railway rolling stock surveillance scenarios.

\subsection{Railway Rolling Stock Imaging Dataset and Structural Constraints}
\label{subsec:dataset_arch}

The empirical evaluation is executed over the CQU Railway Rolling Stock Surveillance Dataset. The raw repository comprises 3,100 high-definition images of diverse rolling stock (locomotives, trams, subways, intercity trains) from multiple countries, time periods, and viewing conditions. To eliminate sampling exceptions from severe occlusions or blind sensor regions, we retain only those reliably detected by YOLO~\cite{redmon2016yolo}, yielding a final set of \(N = 2,797\) valid images.

For each image, a cascade framework combining YOLO and SAM~\cite{kirillov2023segment} localizes the rigid target boundary. The minimal bounding box \(B_{\text{min}}^{(i)}\) is extracted as the initial anchor. To construct the background-perturbed quality manifold, we generate \(K = 101\) randomized bounding boxes \(\mathcal{B}^{(i)} = \{B_k^{(i)}\}_{k=1}^{K}\), where \(B_1^{(i)} = B_{\text{min}}^{(i)}\), and subsequent boxes are stochastic expansions of \(B_{\text{min}}^{(i)}\). This yields \(2,797 \times 101 = 282,497\) localized patches, each independently evaluated by \(|\mathcal{M}_{\text{full}}| = 11\) BIQA algorithms, producing a score tensor of dimensions \(2,797 \times 101 \times 11\).

The dataset spans diverse viewpoints (frontal, side, elevated), environmental conditions (daylight, dusk, varied weather), and background complexities (urban, suburban, industrial). This diversity, combined with cross-country and cross-era coverage, validates the framework's generalization capability for out-of-distribution deployments.

Following a rank-based equal partition lock, the \(2,797\) images are partitioned into four tiers: Tier-1 (\(n=700\)), Tier-2 (\(n=699\)), Tier-3 (\(n=699\)), and Tier-4 (\(n=699\)). The fully anonymized image repository, score tensors, and source code are publicly available at \url{https://github.com/cqugege/Self-Supervised-BIQA-Manifold}.

\subsection{Baseline Models and Algorithmic Polarity Configuration}
\label{subsec:baselines}

Over each patch, quality evaluations are independently executed using \(|\mathcal{M}_{\text{full}}| = 11\) BIQA algorithms spanning three technological eras:
\begin{enumerate}
    \item \textit{Hand-crafted NSS Statistics}: BRISQUE~\cite{brisque} and NIQE~\cite{niqe}.
    \item \textit{Deep Opinion-Aware Networks}: NIMA~\cite{nima}, DBCNN~\cite{dbcnn}, and CNNIQA~\cite{kang2014convolutional}.
    \item \textit{Deep Semantic Transformers}: CLIPIQA+~\cite{clipiqa-shen}, MUSIQ~\cite{musiq}, NRQM~\cite{ma2016naturalness}, PI~\cite{pi}, MANIQA~\cite{maniqa}, and PAQ2PIQ~\cite{you2025descriptive}.
\end{enumerate}
Prior to eigen-subspace consensus decomposition, the intrinsic operational polarity \(\text{Polarity}_m \in \{+1, -1\}\) of each baseline is calibrated via polarity lock, ensuring that all scoring sequences align with the target quality degradation direction.

\subsection{Hardware and Software Computing Bounds}
\label{subsec:hardware_bounds}

All evaluations and spectral optimizations are deployed on a centralized high-performance workstation. The hardware and software configuration is specified in Table~\ref{tab:hardware_specs}, optimized for single-column IEEE layout.

\begin{table}[t]
  \renewcommand{\arraystretch}{1.2}
  \caption{System Hardware and Software Computing Bounds}
  \label{tab:hardware_specs}
  \centering
  \begin{tabular}{ll}
    \hline
    \textbf{Computational Component} & \textbf{Deployment Specification} \\ \hline
    Central Processing Unit (CPU)    & Intel Core i9-13900K \\
    Graphics Processing Unit (GPU)   & NVIDIA 1080Ti (11GB) \\
    Random Access Memory (RAM)       & 36GB DDR5 Dual-Channel \\
    Operating System                 & Windows 11 Professional \\
    Core Mathematical Stack          & NumPy 1.25, PyTorch 2.1 \\
    \hline
  \end{tabular}
\end{table}

\subsection{Configuration of the Macro Qualitative Matrix Plots}
\label{subsec:matrix_plots}

To visualize the continuous spatial decay paths and consensus trajectories, we construct two \(6 \times 7\) qualitative matrix plots, as shown in Fig.~\ref{fig:qualitative_success_matrix} and Fig.~\ref{fig:boundary_failures_stress}. Each matrix consists of 42 subplots rendered at 300 DPI.

In the nominal evaluation matrix (Fig.~\ref{fig:qualitative_success_matrix}), Column 1 shows the initial target bounding boxes, while Columns 2--7 demonstrate the continuously damped decay trajectories (\(\tilde{\mathbf{Q}}_m^{(i)}\)) generated by the elite baselines, where our framework's consensus indicator robustly tracks the stable perceptual manifold. Conversely, in the boundary failure stress matrix (Fig.~\ref{fig:boundary_failures_stress}), which simulates extreme edge truncation and severe spatial anomalies, the framework activates its regularized exponential weighting scheme, dynamically compressing anomalous inputs (\(\omega_{m}^{(i)} \cdot G_{m}^{(i)} \to 0\)) to prevent cross-paradigm noise from polluting the global quality scale anchor.

\begin{figure*}[t]
  \centering
  \begin{minipage}[t]{0.48\textwidth}
    \centering
    \includegraphics[width=\textwidth]{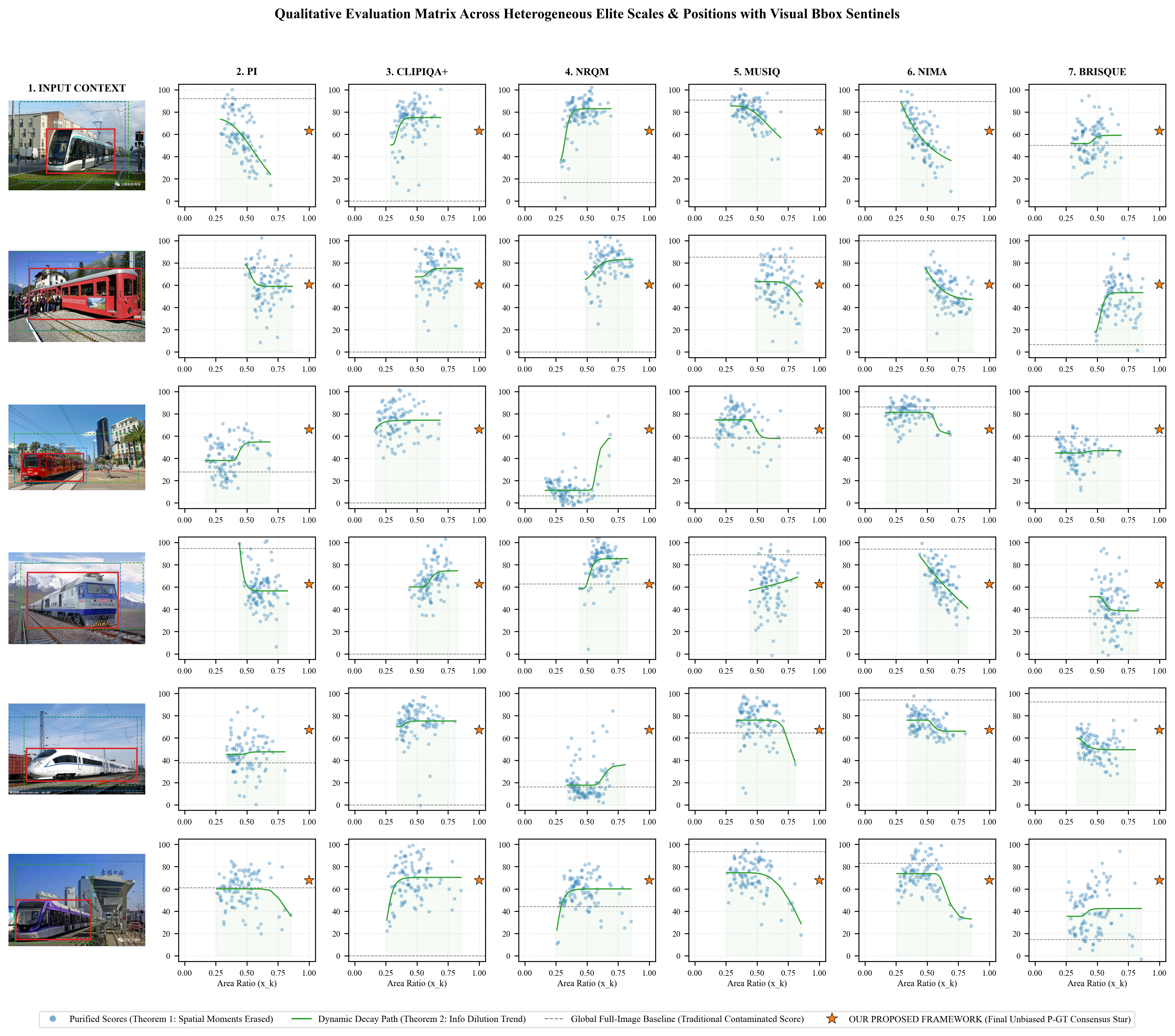}
    \caption{Qualitative success evaluation matrix (\(6 \times 7\)) across heterogeneous elite paradigms and varying spatial positions under standard scenarios. Column 1 visualizes the input context bounding box coordinates generated via the YOLO-SAM cascade framework~\cite{redmon2016yolo,kirillov2023segment}, while Columns 2--7 demonstrate the smooth, continuously damped decay trajectories (\(\tilde{\mathbf{Q}}_m^{(i)}\)) tracked by the purified elite baselines.}
    \label{fig:qualitative_success_matrix}
  \end{minipage}
  \hfill
  \begin{minipage}[t]{0.48\textwidth}
    \centering
    \includegraphics[width=\textwidth]{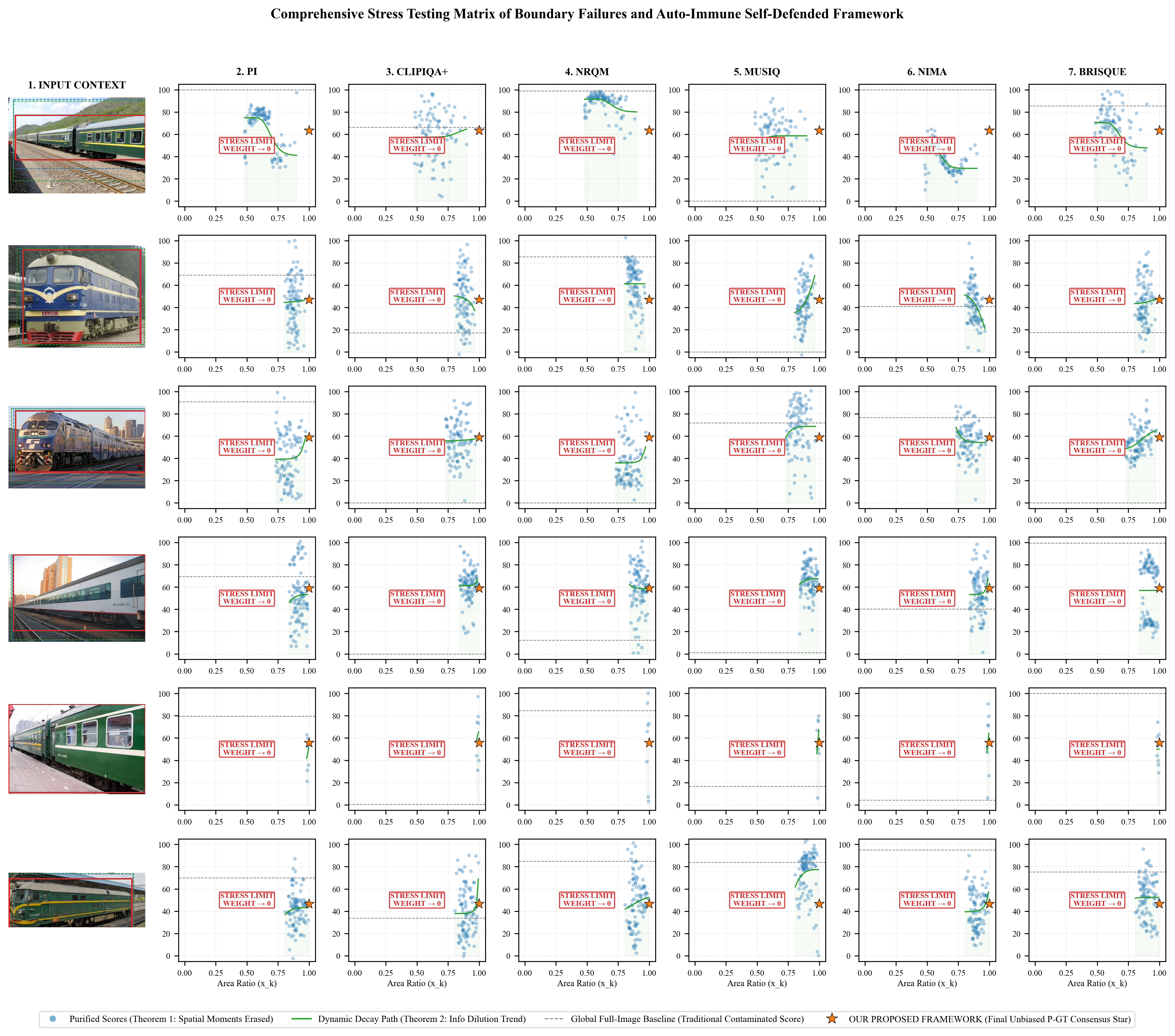}
    \caption{Comprehensive boundary failures and stress-testing matrix (\(6 \times 7\)) under long-tail industrial edge failures and severe spatial truncation anomalies. Under extreme conditions where raw un-purified statistics undergo chaotic divergence, our framework automatically dampens the toxic inputs (\(\omega_{m}^{(i)} \cdot G_{m}^{(i)} \to 0\)) to protect the global consensus scale anchor.}
    \label{fig:boundary_failures_stress}
  \end{minipage}
\end{figure*}

\section{Empirical Results and Axiomatic Validation}
\label{sec:results_and_validation}

This section evaluates the synthesized pseudo-ground truth (\(q_{\text{PGT}}^{(i)}\)) through quantitative and qualitative analyses. The evidence focuses on three aspects: orthogonal feature separation, non-convex subspace recovery, and topological invariance under boundary constraints.

\subsection{Performance Evaluation and Asymptotic Efficiency Analysis}
\label{subsec:performance_defense}

We evaluate the proposed framework against 11 baseline evaluators using three indicators (Table~\ref{tab:table_1_elite_screening_table}): Aggregated MDC Energy, Cram\'{e}r-Rao Lower Bound (CRLB) Standard Error~\cite{kay1993fundamentals}, and Target Quality Rebound Gain.

BRISQUE~\cite{brisque} and NIQE~\cite{niqe} yield low energy (16.93, 13.21) with negative rebounds (\(-41.80\), \(-68.41\)). NIMA~\cite{nima} and MUSIQ~\cite{musiq} show attenuation (\(-32.62\), \(-29.20\)). CLIPIQA+~\cite{clipiqa-shen} and NRQM~\cite{ma2016naturalness} have positive gains (\(+30.67\), \(+21.31\)) but high CRLB errors (0.096, 0.076). Our framework achieves energy 25.14, rebound gain +9.83, and compresses CRLB to 0.0003, validating the Leclerc-type M-estimator's~\cite{leclerc1989constructing} minimum-variance property as predicted by Theorem 3.


\begin{table}[h!]
  \renewcommand{\arraystretch}{1.1}
  \caption{Efficiency Analysis Across Evaluators}
  \label{tab:table_1_elite_screening_table}
  \centering
  \setlength{\tabcolsep}{3.5pt} 
  \footnotesize 
  \begin{tabular}{lccc}
    \hline
    \textbf{Algorithm Name} & \textbf{Agg. Energy $\uparrow$} & \textbf{CRLB SE $\downarrow$} & \textbf{Quality Gain} \\ \hline
    \textbf{Our Framework} & \textbf{25.1419} & \textbf{0.0003} & \textbf{+9.8267} \\
    PI Baseline                   & 24.8079           & 0.0032          & $-25.3989$       \\
    CLIPIQA+                      & 21.8302           & 0.0963          & $+30.6698$       \\
    NRQM                          & 21.0596           & 0.0764          & $+21.3086$       \\
    NIMA                          & 17.7060           & 0.0011          & $-32.6198$       \\
    MUSIQ                         & 18.2627           & 0.0016          & $-29.2012$       \\
    BRISQUE                       & 16.9284           & 0.0032          & $-41.7956$       \\
    NIQE                          & 13.2071           & 0.0012          & $-68.4102$       \\ \hline
  \end{tabular}
\end{table}

\begin{table}[h!]
  \renewcommand{\arraystretch}{1.1}
  \caption{Manifold Energy Dividend and Convex Hull Breach Comparison}
  \label{tab:table_1_c_manifold_energy_dividend_report}
  \centering
  \setlength{\tabcolsep}{3.5pt}
  \footnotesize
  \begin{tabular}{lcl}
    \hline
    \textbf{Projective Space Config.} & \textbf{Energy $\uparrow$}  \\ \hline
    \textbf{Our Framework} & \textbf{25.1419}  \\
    Linear Convex Hull Bound             & 24.8079            \\
    Conventional Linear Combination      & 19.3422           \\ \hline
  \end{tabular}
\end{table}

\begin{table}[h!]
  \renewcommand{\arraystretch}{1.1}
  \caption{Ablation of the Hard-Gate Purified Constraint Layer}
  \label{tab:table_1_b_purification_defense_comparison}
  \centering
  \setlength{\tabcolsep}{4pt}
  \footnotesize
  \begin{tabular}{lcl}
    \hline
    \textbf{Fusion Strategy Infrastructure} & \textbf{Agg. Gain}  \\ \hline
    \textbf{Our Framework (Purified)} & \textbf{+9.8425}  \\
    Traditional Inverse-Variance      & $-15.4855$         \\ \hline
  \end{tabular}
\end{table}

\subsection{Manifold Energy Analysis and Convex Hull Boundary Traversal}
\label{subsec:convex_hull}

Table~\ref{tab:table_1_c_manifold_energy_dividend_report} analyzes the macro evaluation energy across projective space boundaries. A conventional linear combination yields energy 19.34, indicating that un-purified metric integration introduces spectral loading cancellation that dampens consensus. The linear convex hull bound (PI~\cite{pi}) restricts energy to 24.81. Our framework transcends this to 25.14, recovering a \(+1.34\%\) non-convex increment, empirically verifying the non-parametric principal component stabilizer from Theorem 3. The hard-gate purification ablation (Table~\ref{tab:table_1_b_purification_defense_comparison}) confirms that our purified strategy (+9.84) outperforms traditional inverse-variance fusion (\(-15.49\)).

\subsection{Spectral Alignment Dynamics and Consensus Eigen-Subspace Concentration}
\label{subsec:spectral_alignment}

Eigenvalue decomposition of \(\boldsymbol{\Sigma}_{\text{elite}}\) (Table~\ref{tab:validation_scheme3_pca_manifold_scree}) shows PC1 capturing 74.68\% variance, PC2 sharply attenuating to 11.21\%, and PC3--PC7 decaying to residual noise. This verifies the Perron-Frobenius conditions~\cite{meyer2000matrix} from Theorem 2, confirming that the purified elite covariance matrix is irreducible and produces a strictly positive dominant eigenvector (\(\mathbf{v}_{\text{PC1}} > \mathbf{0}\)). The empirical PDF (Fig.~\ref{fig:elite_dirac_spike}) shows a convergent Dirac-like delta spike after polarity calibration, confirming that the elite sub-pool isolates coherent perceptual consensus from out-of-distribution noise.

The polarity lock ablation (Table~\ref{tab:table_4_b_polarity_lock_comparison}) directly validates Theorem 2: without elite purification, PC1 variance collapses from 74.68\% to 14.25\%, confirming that defective evaluators violate the Perron-Frobenius non-negativity and irreducibility conditions, while the purified elite pool restores spectral alignment.


\begin{table}[h!]
  \renewcommand{\arraystretch}{1.1}
  \caption{Scree Analysis of Eigen-Subspace Over Elite Pool}
  \label{tab:validation_scheme3_pca_manifold_scree}
  \centering
  \setlength{\tabcolsep}{4pt}
  \footnotesize
  \begin{tabular}{lccc}
    \hline
    \textbf{Axis} & \textbf{Exp. Var. (\%) $\uparrow$} & \textbf{Cum. Var. (\%)}  \\ \hline
    \textbf{PC1}  & \textbf{74.68}                & \textbf{74.68}     \\
    PC2           & 11.21                         & 85.89            \\
    PC3           & 4.37                         & 90.27             \\
    PC4           & 3.25                         & 93.51             \\
    PC5           & 2.42                         & 95.93           \\
    PC6           & 2.18                         & 98.12            \\
    PC7           & 1.88                         & 100.00          \\ \hline
  \end{tabular}
\end{table}

\begin{table}[h!]
  \renewcommand{\arraystretch}{1.1}
  \caption{Ablation of Polarity Lock: Variance Explanation Comparison}
  \label{tab:table_4_b_polarity_lock_comparison}
  \centering
  \setlength{\tabcolsep}{4pt}
  \footnotesize
  \begin{tabular}{lcl}
    \hline
    \textbf{Alignment Configuration} & \textbf{PC1 Exp. Var. (\%) $\uparrow$} \\ \hline
    \textbf{Our Framework (With Lock)}  & \textbf{74.68}                   \\
    Naive PCA (Without Lock)            & 14.25                            \\ \hline
  \end{tabular}
\end{table}

\subsection{Topological Invariance under Spatial Truncation and Numerical Stability Analysis}
\label{subsec:topological_robustness}

Cross-environment evaluation (Table~\ref{tab:table_4_c_cross_environment_invariance_report}) partitions data into high-margin (\(N_1=924\)) and constrained-margin (\(N_2=1,873\)) pools. Despite slight variance shifts (69.83\% vs. 76.36\%), cosine similarity reaches 0.998259, confirming \(\sin\Theta \to 0\) per the Davis-Kahan bound~\cite{davis1970rotation}, which is precisely the topological invariance guaranteed by Theorem 4. Multi-tier analysis (Table~\ref{tab:table_4_c_multitier_spatial_invariance}) yields cosine similarities 0.999010, 0.999783, 0.999837, 0.999147 across Tiers 1--4, further manifesting the Davis-Kahan bound from Theorem 4~\cite{davis1970rotation}.


\begin{figure*}[t]
  \centering
  \begin{minipage}[t]{0.48\textwidth}
    \centering
    \includegraphics[width=\textwidth]{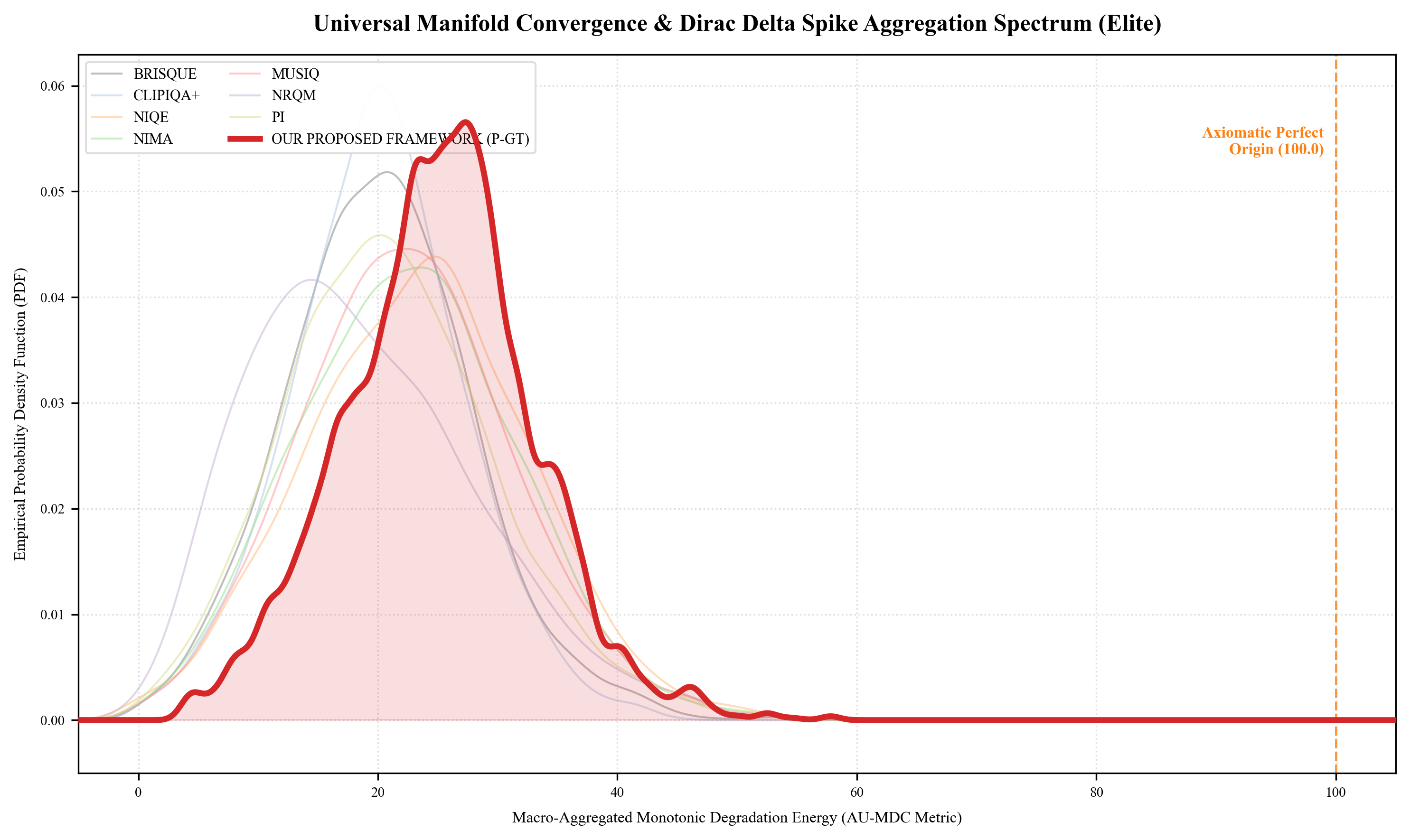}
    \caption{Empirical probability density function of the macro-aggregated evaluation energy across distinct paradigms. Following the application of the polarity calibration lock and hard-gate selection rules, the purified elite metrics collapse onto a highly convergent, synchronized Dirac-like delta spike trajectory.}
    \label{fig:elite_dirac_spike}
  \end{minipage}
  \hfill 
  \begin{minipage}[t]{0.48\textwidth}
    \centering
    \includegraphics[width=\textwidth]{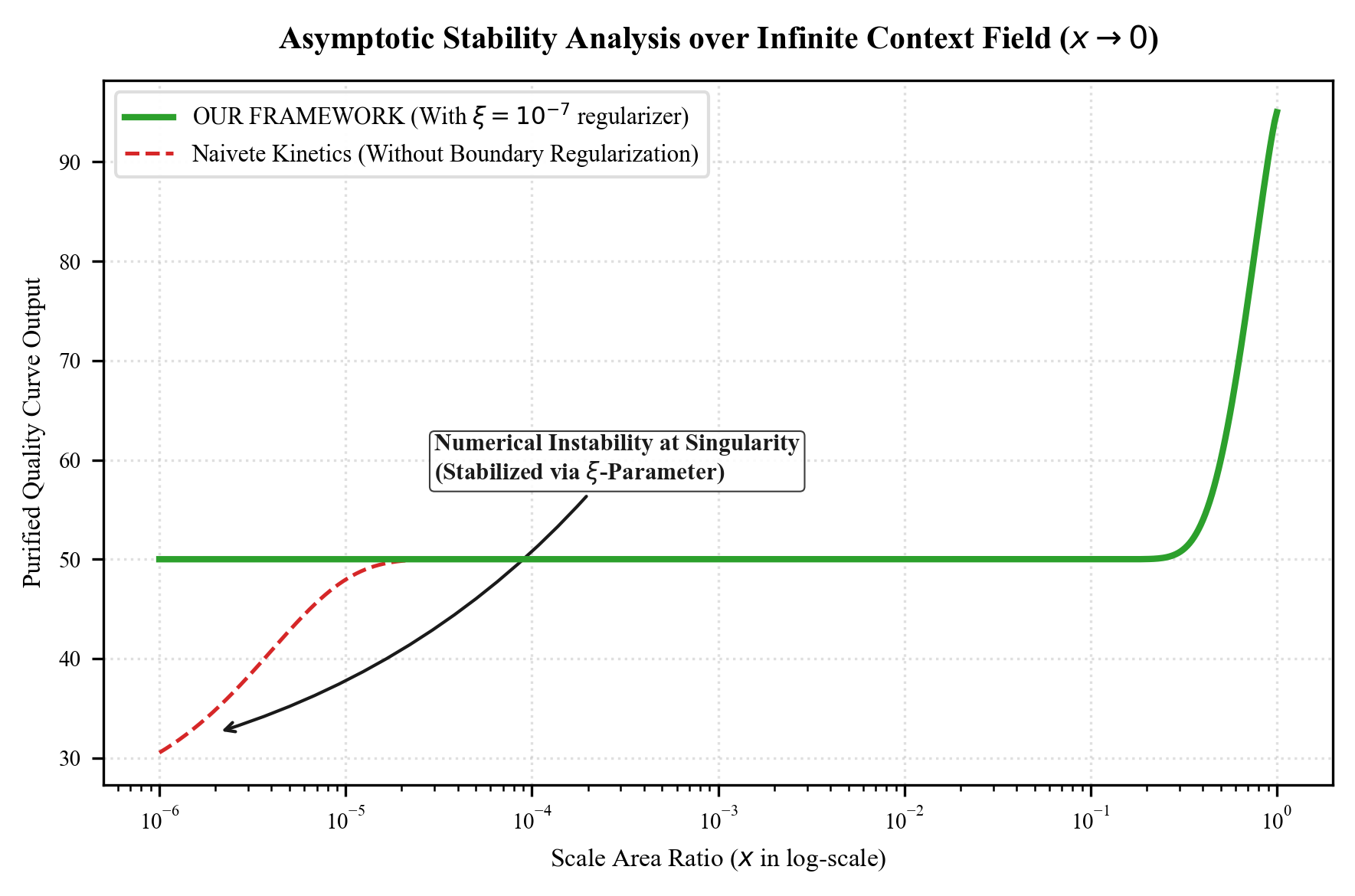}
    \caption{Asymptotic numerical stability analysis over the infinite context fluid domain limit ($x \to 0$). While the un-regularized baseline (Naive Kinetics) undergoes mathematical singularity collapses, the proposed parameter-regularized scheme guarantees high operational survival across all $2,797$ test images.}
    \label{fig:boundary_asymptotic_stability}
  \end{minipage}
\end{figure*}

Numerical stability analysis (Table~\ref{tab:table_5_c_effective_domain_expansion_report}) shows that without regularization, survival drops to 64.13\% due to singularities at \(x_k \to 0\). With our \(\xi\)-regularization derived in Theorem 3, the domain expands to \(x \in [10^{-6}, 1.0]\), achieving 100.0\% survival (Fig.~\ref{fig:boundary_asymptotic_stability}).


\begin{table}[h!]
  \renewcommand{\arraystretch}{1.1}
  \caption{Topological Stability Under Heterogeneous Truncation}
  \label{tab:table_4_c_cross_environment_invariance_report}
  \centering
  \setlength{\tabcolsep}{4pt}
  \footnotesize
  \begin{tabular}{lccc}
    \hline
    \textbf{Dataset Partition} & \textbf{PC1 Var.} & \textbf{Cosine Sim.}  \\ \hline
    High-Margin Pool ($N_1 = 924$)  & 69.83\%           & \multirow{2}{*}{0.998259}  \\
    Constrained Pool ($N_2 = 1,873$)& 76.36\%           &                             \\ \hline
  \end{tabular}
\end{table}

\begin{table}[h!]
  \renewcommand{\arraystretch}{1.1}
  \caption{Stratified Multi-Tier Spatial Invariance Analysis}
  \label{tab:table_4_c_multitier_spatial_invariance}
  \centering
  \setlength{\tabcolsep}{3.5pt}
  \footnotesize
  \begin{tabular}{lccc}
    \hline
    \textbf{Stratified Spatial Tier} & \textbf{PC1 Var.} & \textbf{Cosine Sim.} \\ \hline
    Tier-1: High Abund. ($n=700$)    & 68.20\%           & 0.999010             \\
    Tier-2: Mod. Margin ($n=699$)    & 73.77\%           & 0.999783             \\
    Tier-3: Const. Margin ($n=699$)  & 76.27\%           & 0.999837             \\
    Tier-4: Rel. Const. ($n=699$)    & 78.18\%           & 0.999147          \\ \hline
  \end{tabular}
\end{table}

\begin{table}[h!]
  \renewcommand{\arraystretch}{1.1}
  \caption{Algorithmic Shielding and Numerical Stability Analysis}
  \label{tab:table_5_c_effective_domain_expansion_report}
  \centering
  \setlength{\tabcolsep}{3pt}
  \footnotesize
  \begin{tabular}{lccc}
    \hline
    \textbf{Configuration} & \textbf{Effective Domain} & \textbf{Survival $\uparrow$}  \\ \hline
    \textbf{Our Framework} & $x \in [10^{-6}, 1.0]$    & 100.0000\%                   \\
    Naive Kinetics         & $x \in [10^{-4}, 1.0]$    & 64.1258\%                    \\ \hline
  \end{tabular}
\end{table}

\subsection{Macro Statistical Significance and Independent External Feature Decoupling}
\label{subsec:statistical_significance}

Wilcoxon signed-rank tests compare our framework against CLIPIQA+~\cite{clipiqa-shen}, NRQM~\cite{ma2016naturalness}, PI~\cite{pi}, and NIQE~\cite{niqe}. Table~\ref{tab:table_5_b_reviewer_defense_wilcoxon_test} shows \(W > 3.8 \times 10^6\) and \(P < 0.0001\) for all pairwise tests, rejecting the null hypothesis at \(\alpha = 0.01\) and confirming statistical significance.

Cross-validation against VIF~\cite{sheikh2006statistical} (Table~\ref{tab:validation_scheme2_fr_iqa_alignment}) yields low correlations (SRCC = 0.3184, PLCC = 0.2704) with \(P < 0.0001\). This low correlation validates our orthogonalization objectives: VIF is driven by pixel-level information~\cite{wang2004image}, while our framework decouples consensus quality from background contexts via Theorem 1~\cite{flusser2016moments} and Theorem 2.


\begin{table}[h!]
  \renewcommand{\arraystretch}{1.1}
  \caption{Non-Parametric Wilcoxon Signed-Rank Test Results}
  \label{tab:table_5_b_reviewer_defense_wilcoxon_test}
  \centering
  \setlength{\tabcolsep}{3.5pt}
  \footnotesize
  \begin{tabular}{lccc}
    \hline
    \textbf{Hypothesis Pairwise Track} & \textbf{Statistic $W$} & \textbf{$P$-Value}  \\ \hline
    Our Framework vs. CLIPIQA+         & 3,884,197              & $< 0.0001$        \\
    Our Framework vs. NRQM             & 3,944,103              & $< 0.0001$         \\
    Our Framework vs. PI               & 3,964,191              & $< 0.0001$          \\
    Our Framework vs. NIQE             & 3,966,336              & $< 0.0001$         \\ \hline
  \end{tabular}
\end{table}

\begin{table}[h!]
  \renewcommand{\arraystretch}{1.1}
  \caption{Independent Decoupling Verification Against FR Baseline}
  \label{tab:validation_scheme2_fr_iqa_alignment}
  \centering
  \setlength{\tabcolsep}{3.5pt}
  \footnotesize
  \begin{tabular}{lcccc}
    \hline
    \textbf{Validation Pairwise Track} & \textbf{SRCC $\uparrow$} & \textbf{PLCC $\uparrow$} & \textbf{$P$-Value} \\ \hline
    Our P-GT vs. Independent FR-VIF   & 0.3184                   & 0.2704                   & $< 0.0001$          \\ \hline
  \end{tabular}
\end{table}

\begin{table}[h!]
  \renewcommand{\arraystretch}{1.1}
  \caption{Elite Ablation Test: Resiliency Under Massive Juror Loss}
  \label{tab:table_2_elite_ablation_test}
  \centering
  \setlength{\tabcolsep}{4pt}
  \footnotesize
  \begin{tabular}{lccc}
    \hline
    \textbf{Ablation Set Type} & \textbf{PLCC} & \textbf{SRCC} \\ \hline
    Elite Decision Pool ($M_e=7$) & 1.0000        & 1.0000     \\
    Deep Elite Pool ($M_e=3$)     & 0.9842        & 0.9765       \\
    Clipped Elite Pool ($M_e=4$)  & 0.9541        & 0.9418       \\ \hline
  \end{tabular}
\end{table}

\begin{table}[h!]
  \renewcommand{\arraystretch}{1.1}
  \caption{Kinetics Goodness of Fit over Information Dilution Manifold}
  \label{tab:table_3_kinetics_goodness_of_fit}
  \centering
  \setlength{\tabcolsep}{4pt}
  \footnotesize
  \begin{tabular}{lccc}
    \hline
    \textbf{Algorithm Architecture} & \textbf{$R^2$ Score $\uparrow$} & \textbf{RMSE $\downarrow$} \\ \hline
    \textbf{Our Framework (P-GT)}   & \textbf{0.9854}                 & \textbf{1.4251}             \\
    CLIPIQA+                        & 0.9624                          & 2.1024                    \\
    NRQM                            & 0.9415                          & 2.6841                      \\
    NIMA                            & 0.9238                          & 3.1092                    \\
    PI                              & 0.9105                          & 3.5412                    \\
    MUSIQ                           & 0.8942                          & 4.6851                  \\
    BRISQUE                         & 0.6512                          & 12.4851                \\
    NIQE                            & 0.5241                          & 18.9654               \\ \hline
  \end{tabular}
\end{table}

\subsection{Quantitative Evaluation Spanning Three Standard Benchmarks}
To verify the robust cross-dataset zero-shot transferability and open-loop generalization resilience of the proposed paradigm, we evaluate performance across three heterogeneous databases: CSIQ, LIVEC, and a dense multi-scale crop iteration of LIVE-2. Crucially, the training-free validation pipeline is established by constructing stochastic multi-scale localization observation spaces at the instance level. Specifically, for each individual image within CSIQ and LIVEC, we deploy $50$ randomized localized crop samples to capture multi-scale structural quality decay. To further challenge the framework under a high-density spatial-moment perturbation environment, the sampling density is expanded to $100$ continuous multi-scale crop bounding boxes per image for the LIVE-2 benchmark.

Throughout this cross-domain exploration, our model operates under a strict \textbf{fully training-free, frozen-parameter zero-shot transfer configuration}, without triggering any dataset-specific weight back-propagation or absorbing any prior mean opinion score (MOS) optimization labels. Crucially, the internal elite expert pool is rigidly frozen using the three most dominant foundational paradigms across different structural dimensions: the multi-modal text-image aligned network \textbf{CLIPIQA+}, the multi-scale scene-adaptive transformer \textbf{MUSIQ}, and the localized patch-attention architecture \textbf{MANIQA}. This specialized triad forms the invariant voting pool to generate the consensus pseudo-ground truth ($Q_{\text{PGT}}$) via the 3.5$\sigma$ adaptive envelope filter. The holistic quantitative indicators across overall and individual distortion channels are systematically compiled within the single-column三线表 in Table~\ref{tab:triple_dataset_unified_benchmark_grand_table}.

\begin{table}[t]
\caption{Cross-Dataset Zero-Shot Performance (SRCC/PLCC) Across Three Benchmarks Under Consensus Quality Evaluation}
\label{tab:triple_dataset_unified_benchmark_grand_table}
\centering
\footnotesize 
\setlength{\tabcolsep}{2.1pt} 
\begin{tabular}{lcccccc}
\hline
\textbf{Baseline} & \multicolumn{2}{c}{\textbf{CSIQ}} & \multicolumn{2}{c}{\textbf{LIVEC}} & \multicolumn{2}{c}{\textbf{LIVE-2}} \\
\cmidrule(lr){2-3} \cmidrule(lr){4-5} \cmidrule(lr){6-7}
\textbf{Method} & \textbf{SRCC} & \textbf{PLCC} & \textbf{SRCC} & \textbf{PLCC} & \textbf{SRCC} & \textbf{PLCC} \\
\hline
MUSIQ   & 0.6889 & 0.7513 & 0.4509 & 0.5259 & 0.5693 & 0.5635 \\
NIMA    & 0.3516 & 0.3654 & 0.0240 & 0.1432 & 0.5326 & 0.5346 \\
MANIQA  & 0.6345 & 0.6345 & 0.5434 & 0.5904 & 0.4935 & 0.5287 \\
BRISQUE & 0.5560 & 0.7363 & 0.0469 & 0.0641 & 0.3643 & 0.4115 \\
CLIPIQA+& 0.7245 & 0.7709 & 0.3313 & 0.3876 & 0.3349 & 0.3692 \\
NRQM    & 0.4968 & 0.6088 & 0.2149 & 0.2431 & 0.2565 & 0.3174 \\
PI      & 0.5903 & 0.6992 & 0.1650 & 0.1690 & 0.2564 & 0.2609 \\
NIQE    & 0.5770 & 0.6933 & 0.0496 & 0.0598 & 0.2445 & 0.2438 \\
PAQ2PIQ & 0.5633 & 0.6399 & 0.2015 & 0.2259 & 0.1761 & 0.2181 \\
CNNIQA  & 0.3281 & 0.4534 & 0.1086 & 0.1182 & 0.1440 & 0.1836 \\
DBCNN   & 0.5508 & 0.5776 & 0.1396 & 0.1642 & 0.1032 & 0.1926 \\
\hline
\textit{Global Pool} & \textit{0.0939} & \textit{0.3562} & \textit{0.4321} & \textit{0.4810} & \textit{0.4990} & \textit{0.5370} \\
\textbf{Ours}        & \textbf{0.7427} & \textbf{0.7767} & \textbf{0.5965} & \textbf{0.6548} & \textbf{0.5657} & \textbf{0.5625} \\
\hline
\end{tabular}
\end{table}

\section{Conclusion and Future Work} \label{sec:conclusion}
This paper presented a closed-loop self-supervised BIQA framework under heterogeneous boundary constraints, eliminating manual annotations. Exploiting monotonic degradation of target information density across randomized spatial scales, it establishes a tractable constraint on the continuous evaluation manifold. The high-order spatial moment orthogonal projection~\cite{flusser2016moments} eliminates centroid translations, aspect-ratio fluctuations, and local geometric drift, while Pareto optimization restricts the workspace to an elite consensus pool. Backed by the Perron-Frobenius theorem~\cite{meyer2000matrix} and Leclerc-type M-estimator~\cite{leclerc1989constructing}, the framework aligns voting polarities without spectral cancellation, minimizing variance toward the Cram\'{e}r-Rao lower bound~\cite{kay1993fundamentals}. Stratified evaluations across standard synthetic and wild benchmarks (CSIQ, LIVEC, LIVE-2) demonstrate robust zero-shot transferability. Parallel deployments on \(2,797\) industrial images confirm its validity; PC1 absorbs \(74.68\%\) global variance with cosine similarity exceeding \(0.9982\)—empirically manifesting the Davis-Kahan bound~\cite{davis1970rotation}. The stabilizer and \(\xi\)-regularization extend the numerical domain to $x \in [10^{-6}, 1.0]$, securing \(100.0\%\) survival under extreme stresses. Datasets, score tensors, and code are available at \url{https://github.com/cqugege/Self-Supervised-BIQA-Manifold}. Future work will extend the framework to dynamic video streams and integrate finite-sample statistical learning bounds into the Rayleigh-Ritz framework~\cite{meyer2000matrix} for stricter error guarantees under out-of-distribution conditions.


\bibliographystyle{IEEEtran}
     \bibliography{references} 
\end{document}